\documentclass{article}

\usepackage[main, final]{neurips_2026}

\usepackage[utf8]{inputenc} 
\usepackage[T1]{fontenc}    
\usepackage{hyperref}       
\usepackage{url}            
\usepackage{booktabs}       
\usepackage{amsfonts}       
\usepackage{nicefrac}       
\usepackage{microtype}      
\usepackage{xcolor}         

\usepackage{url}            
\usepackage{booktabs}       
\usepackage{amsfonts}       
\usepackage{nicefrac}       
\usepackage{microtype}      
\usepackage[dvipsnames, table]{xcolor}
\definecolor{estcolor}{RGB}{0,102, 128}

\usepackage{amsmath}
\usepackage{amssymb}
\usepackage{mathtools}
\usepackage{amsthm}
\usepackage{bbm}
\usepackage{nicefrac} 
\usepackage{natbib}
\usepackage{tikz}
\usetikzlibrary{fit, positioning}

\RequirePackage{algorithm}
\usepackage{algpseudocode}
\usepackage{wrapfig}
\usepackage{caption}
\usepackage{needspace}
\usepackage{subcaption}
\usepackage{caption}

\newcommand\p{p}

\newcommand\phat{\hat{p}}
\newcommand\qhat{\hat{q}}
\newcommand\E{\mathbb{E}}

\newcommand\KL{\mathrm{KL}}

\newcommand\D{\mathcal{D}}
\renewcommand\L{\mathcal{L}}
\newcommand\Lhat{\hat{\mathcal{L}}}

\newcommand\Etr{\mathcal{E}_{\mathrm{tr}}}
\newcommand\diag{\mathrm{diag}}

\theoremstyle{plain}
\newtheorem{theorem}{Theorem}
\newtheorem{proposition}[theorem]{Proposition}

\newtheorem*{proposition*}{Proposition}

\usepackage{cleveref}
\crefname{assumption}{assumption}{assumptions}
\Crefname{assumption}{Assumption}{Assumptions}

\crefname{proposition}{proposition}{propositions}
\Crefname{proposition}{Proposition}{Propositions}

\crefname{definition}{definition}{definitions}
\Crefname{definition}{Definition}{Definitions}

\crefname{appendix}{appendix}{appendices}
\Crefname{appendix}{Appendix}{Appendices}

\crefname{section}{section}{sections}
\Crefname{section}{Section}{Sections}

\hypersetup{hidelinks}

\title{Environment-Robust Representation Learning \\ with Empirical Bayes}

\author{
  Yuli Slavutsky\\
  Department of Statistics\\
  Columbia University\\
  New York, NY 10027, USA \\
  \texttt{ys3938@columbia.edu} \\
  \And
  Matthew Shen \\
  Department of Statistics\\
  Columbia University\\
  New York, NY 10027, USA \\
  \texttt{ms7079@columbia.edu} \\
  \And
  Bohan Wu \\
  Department of Statistics\\
  Columbia University\\
  New York, NY 10027, USA \\
  \texttt{bw2766@columbia.edu} \\
  \And
  David M. Blei \\
  Departments of Statistics and Computer Science\\
  Columbia University\\
  New York, NY 10027, USA \\
  \texttt{david.blei@columbia.edu} \\
}

\begin{document}

\maketitle

\vspace{-1em}
\begin{abstract}
  We consider multi-environment prediction problems. We assume the environments change the distribution of a latent variable, while the mechanisms generating observed covariates and targets remain stable conditional on that variable. For example, hospitals or clinical cohorts may differ in the prevalence of latent patient states, even though the relationships between those states, physiological measurements, and outcomes remain unchanged. 
  Given a dataset from multiple environments, we formulate a Bayesian model for such problems and derive the corresponding variational objective. We show that this objective decomposes into per-environment terms and an additional cross-environment balancing term induced by the model's structure. We use an empirical Bayes method to set the prior and incorporate it into the objective. Based on this objective, we develop an amortized variational algorithm for posterior approximation, and use the resulting learned latent variables to form predictions in new environments.
  We study our approach through simulations and real-world studies of astronomical source identification, microbiome-based disease detection, and ICU sepsis prediction. Across these settings, our method outperforms previous approaches for prediction in new environments.
\end{abstract}

\section{Introduction} \label{sec:intro}

Scientific prediction problems often involve data collected across heterogeneous environments. In astronomical source identification, for example, we observe measurements of objects across different regions of the sky, and the goal is to predict the object type. 
Different sky regions can have different prevalence of latent underlying source types \citep{humphrey2022machine}, while the measurement process and the relationship between source types and labels are stable across regions. In microbiome studies, we observe microbial abundance profiles from patients across different body habitats or study cohorts, and the goal is to predict disease. Different environments can induce different mixture proportions over latent microbial states \citep{human2012framework}, while sequencing and labeling mechanisms are stable across cohorts and habitats.

Motivated by these examples, we develop a probabilistic model 
for multi-environment prediction \citep{peters2016causal}, where each environment captures variation in the distribution of latent variables. Formally, we assume that environments are sampled from $p(e)$ and for each environment, latent variables are sampled from $p(z\mid e)$. An outcome of interest $y$ depends on $z$ through $p(y\mid z)$. We observe a potentially high-dimensional proxy variable $x$, sampled according to $p(x\mid z)$, where the mechanisms $p(x\mid z)$ and $p(y\mid z)$ remain stable across environments.
This corresponds to the probabilistic graphical model in Figure \ref{fig:pgm}.

We observe a training dataset
$\D = \{\{(x_i^e, y_i^e)\}_{i = 1}^{n_e}\}_{e \in \Etr}$, 
drawn across a finite set of training environments 
$\Etr=\{e_1, \dots, e_m\}$. The probability of this dataset is
\begin{equation} \label{eq:Pexy}
\p(\D) = \prod_{j=1}^{m} \p(e_j)\,
\prod_{i=1}^{n_{e_j}}
\int
\p(z_{ij} \mid e_j)\,
\p(y_{ij} \mid z_{ij})\,
\p(x_{ij} \mid z_{ij})\,
dz_{ij}.
\end{equation}

Our goal is to predict an outcome $y$ for a new observation $x$, drawn according to $p(x \mid e=e')$, where $e'$ is a new environment that was not observed during training. 
Since $e' \not \in \Etr$, the environment-specific Bayes target $p(y\mid x,e')=\int p(y\mid z)\,p(z\mid x,e')\,dz$
cannot in general be recovered.
For this reason, we target the marginal Bayes rule $p(y \mid x)$, where 
\begin{equation} \label{eq:marginal}
    p(y\mid x)=\iint  p(z,e\mid x)\, p(y\mid x,z,e) \,dz\,de=\iint p(z\mid x, e)\,p(e\mid x)\, p(y\mid z) \, dz\,de.
\end{equation}

\vspace{-0.5em}Our target, namely estimating $p(y\mid x)$ for an observation $x$ from an unobserved environment, requires estimating three terms.
The first term is 
\begin{equation}
    p(z \mid x,e) = \int p(y \mid x, e)\, p(z \mid x,y,e) \, dy.
\end{equation}
In the astronomical source classification example, it corresponds to the distribution of latent source types $z$ given object measurements $x$,  and sky region $e$.
Estimating this term requires averaging over the possible labels for the same observation $x$. In the data, however, each observed $x$ is typically paired with only one realization of $y$. 
Consequently, the latent information directly inferred from an observed pair is label-conditioned: it corresponds to $p(z \mid x,y,e)$, rather than to the desired marginal quantity $p(z \mid x,e)$.

Note, however, that $p(y \mid x, e)$ does not include the latent $z$, which allows directly estimating it from the data. 
We approximate  $p(z \mid x,y,e)$ and obtain $p(z \mid x,e)$ by marginalizing over 
$p(y \mid x, e)$. This is an empirical-Bayes \citep{Efron2012Large-ScalePrediction} step: the prior distribution $p(z \mid x, e)$, which depends on the known covariates $x$ but not $y$, is estimated using observed data.

The second term is $p(e\mid x)$. In our astronomical source classification example, it measures how plausible each observed sky region $e$ is for measurements $x$.
A naive estimate is the training indicator of whether $x$ was observed in environment $e$, but this indicator is unavailable for observations from new environments. 
Such estimates are not observation-specific: they do not adapt to a new input $x$, and  cannot estimate how likely each training environment is for that observation.
To address this, we learn a predictor of $p(e\mid x)$ that estimates the probability of each training environment given the observation $x$.

The third term is the likelihood of each label $y$ given a latent  state $z$, that is $p(y \mid z)$. To this end, we introduce an amortized variational neural model in which $z$ is inferred from the observed data, while $p(y \mid z)$ is parameterized directly as the conditional label distribution given the latent state.

\begin{wrapfigure}[15]{r}{0.53\textwidth}
    \vspace{-1.5em}
    \centering
    \resizebox{\linewidth}{!}{%
    \begin{tikzpicture}
      \tikzset{
        var/.style={circle, draw=black, thin, minimum size=7mm, inner sep=0pt, font=\small},
        grayvar/.style={var, fill=gray!20},
        tinyvar/.style={circle, draw=black, thin, minimum size=2mm, inner sep=0pt},
        plate/.style={draw, rounded corners}
      }

      \node[tinyvar] (u) at (2.1,5.85) {};

      \node[grayvar] (e) at (0,5) {$e_j$};
      \node[var]     (z) at (0,3.6) {$z_{ij}$};
      \node[grayvar] (x) at (-1.1,3) {$x_{ij}$};
      \node[grayvar] (y) at (1.1,3) {$y_{ij}$};

      \draw[->] (u) -| (e);
      \draw[->] (e) -- (z);
      \draw[->] (z) -- (x);
      \draw[->] (z) -- (y);

      \node[
        plate,
        fit=(z)(x)(y),
        inner sep=8pt,
        label={[xshift=-29pt,yshift=-14pt]north east:\footnotesize{$1:n_{e_j}$}}
      ] (innerplate) {};

      \node[
        plate,
        fit=(e)(innerplate),
        inner sep=8pt,
        label={[xshift=-26pt,yshift=-12pt]north east:\footnotesize{$1:m$}}
      ] (outerplate) {};

      \node[var]     (ep) at (3.8,5) {$e$};
      \node[var]     (zp) at (3.8,3.6) {$z$};
      \node[grayvar] (xp) at (2.7,3) {$x$};
      \node[var]     (yp) at (4.9,3) {$y$};

      \draw[->] (u) -| (ep);
      \draw[->] (ep) -- (zp);
      \draw[->] (zp) -- (xp);
      \draw[->] (zp) -- (yp);

      \node[font=\small] at (0,1.8) {Observed data};
      \node[font=\small] at (3.8, 1.8) {Prediction};
    \end{tikzpicture}
    }
    \caption{Probabilistic graphical model of the multi-environment data generative process. Observed variables shown in gray. The top dot encodes shared distribution across environments.}
    \label{fig:pgm}
    \vspace{-3em}
\end{wrapfigure}
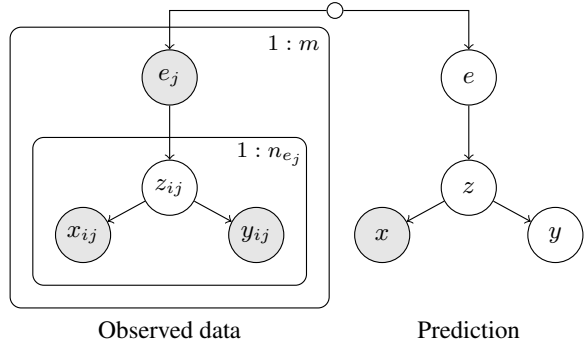
Together, these ideas form the Empirical-Bayes Environment-Robust (EmBER) method. In what follows, we derive a variational objective and the resulting estimation algorithm. 

We then show that EmBER can be interpreted as an \emph{environment-balancing} method, namely a method whose objective combines environment-wise terms with a cross-environment regularization penalty. Unlike  standard environment-balancing methods, this penalty is not pre-specified, but is induced by the probabilistic model. 

Finally, we evaluate EmBER on synthetic and real-world datasets, and show that it often improves predictive performance relative to other environment-balancing methods.

\subsection{Related Work} \label{sec:related}

A central line of work in out-of-distribution generalization seeks representations for which the predictive relationship with the target remains stable across environments. These methods typically optimize an objective of the form
\begin{equation} \textstyle
\min_{\theta}
\;
\sum_{e\in\Etr}
\ell_e(f_\theta)
+
\lambda\,\mathcal R(f_\theta,\Etr),
\end{equation}
where $\ell_e(f_\theta)$ is the loss of the representation $f_\theta$ in environment $e$, and $\mathcal R$ encourages some notion of similar behavior across environments. This perspective underlies invariant representation learning methods such as invariant risk minimization and its variants \citep{IRM, ahuja2020invariant, lu2021invariant,  lin2022bayesian}, causality-oriented robust methods \citep{shen2023causality}, as well as invariant approaches to transfer learning \citep{rojas2018invariant}. Closely related methods, including risk extrapolation and calibration-based objectives, also fit this framework by favoring predictors whose risks or calibration levels remain stable across environments \citep{krueger2021out, wald2021calibration,shi2021gradient}. 

In \S \ref{sec:balancing} we show that our approach belongs to this category of methods. However, rather than predefining a notion of stability across environments and enforcing it through a chosen penalty, we derive the balancing term from the assumed latent-variable model.

A complementary direction is to model environment-specific variation explicitly, rather than penalizing it away. Under different probabilistic assumptions, \citet{slavutsky2026robust} show that explicitly modeling the effect of the environment can improve robust prediction.

Other related lines of work target different notions of stability. These include distributionally robust optimization \citep{DRO, duchi2021statistics, duchi2021learning, wei2023distributionally} and its group variants \citep{sagawa2019distributionally, piratla2021focus}, which optimize performance under worst-case perturbations, rather than the average performance over new environments.

Another related line is latent-variable representation learning across multiple conditions. Several methods based on variational autoencoders \citep{vae} separate shared and condition-specific factors \citep{CVAE, CSVAE, scVI, CCVAE}.  More recently, \citet{wu2026multi} used empirical Bayes for multi-domain causal representation learning under a linear mixing model. However, these approaches impose independence, causal, or prior assumptions that do not match our assumed generative process. 

In Appendix \ref{sup:related} we discuss broader connections beyond the most directly related literature, including causal representation learning beyond environment-balancing criteria, connections to mixed-effects models, and worst-case robustness approaches.

\section{Multi-Environment Variational Learning}

Our goal is to estimate the marginal predictive distribution in Equation \eqref{eq:marginal}.

For the true probability $p$, under the probabilistic graphical model in Figure \ref{fig:pgm} we have that 
\begin{equation}\label{eq:psuedo-elbo}
    \log p(y \mid x) = \E_{p(z \mid x,y)}[\log p(y \mid z)]
    -\KL(p(z \mid x,y)\, \Vert \, p(z \mid x)),
\end{equation}
where 
\begin{align}
\label{eq:posterior}
 & p(z\mid x,y) = \int p(z\mid x,y,e)\,p(e\mid x,y)\,de\\ 
 \label{eq:prior}
 & p(z\mid x) = \int p(z\mid x, e)\,p(e\mid x)\,de. 
\end{align}

Since by Bayes' rule $p(e \mid x, y) \propto p(y \mid x,e) \, p(e \mid x)$, it suffices to estimate (i) the posterior $p(z\mid x,y,e)$, (ii) the conditional distribution of the label given the latent variable $p(y \mid z)$, 
(iii) the conditional distribution of the label given the observation and environment $p(y \mid x,e)$, and 
(iv) the conditional distribution of the environment $p(e \mid x)$. 

Below, we formulate a single objective to estimate all these terms.

Let $q(z ; x,y,e)$ be an environment-specific variational approximation to $p(z \mid x,y,e)$. Using the same marginalization identities as in Equations \eqref{eq:posterior}--\eqref{eq:prior}, this variational posterior induces the global posterior $q(z ; x,y)$ and the label-marginal prior $q(z ; x)$.

Given these induced distributions, the corresponding conditional distribution of $y$ given $x$ is
\begin{equation}
    p_q(y \mid x)
    \coloneqq
    \int p(y \mid z)\, q(z ; x) \, dz .
\end{equation}

Our goal is to maximize the conditional log-likelihood 
\begin{equation}
    \frac{1}{\vert \D \vert}
    \sum_{e \in \Etr}
    \sum_{i=1}^{n_e}
    \log p_q(y_i^e \mid x_i^e),
\end{equation}
where, for each observation,
\begin{equation}
    p_q(y_i^e \mid x_i^e)
    \coloneqq
    \int p(y_i^e \mid z) \, q(z ; x_i^e) \, dz .
\end{equation}
The conditional log-likelihood is generally intractable. However, for any variational family such that
$q(z ; x_i^e,y_i^e)$ is absolutely continuous with respect to $q(z ; x_i^e)$, by Jensen's inequality\footnote{This is a slight variation of the standard ELBO derivation.} we have that 
\begin{align} \notag
    \log p_q(y_i^e \mid x_i^e)
    & =
    \log
    \int
    p(y_i^e \mid z)\, q(z ; x_i^e)\, dz
    =
    \log
    \mathbb E_{q(z ; x_i^e,y_i^e)}
    \left[
        \frac{
            p(y_i^e \mid z)\, q(z ; x_i^e)
        }{
            q(z ; x_i^e,y_i^e)
        }
    \right] \\
    & \geq 
    \mathbb E_{q(z ; x_i^e,y_i^e)}
    \left[
        \log p(y_i^e \mid z)
    \right]
    -
    \KL
    \left(
        q(z ; x_i^e,y_i^e)
        \, \Vert \,
        q(z ; x_i^e)
    \right).
    \label{eq:pq_elbo_single}
\end{align}

Consequently, each term in the empirical conditional log-likelihood admits the lower bound in Equation~\eqref{eq:pq_elbo_single}.

We now specify the parameterized estimators that define the distributions appearing in this bound.

\paragraph{Parameterized Estimators}

Rather than optimizing separate variational parameters for each observation, we model the environment-specific variational posterior $q(z ; x,y,e)$ using an amortized neural network $g_\theta$, parameterized by $\theta$. This network maps $(x,y,e)$ to the mean and diagonal covariance entries of a Gaussian distribution, $g_\theta(x,y,e) = (\mu_\theta(x,y,e),\sigma_\theta(x,y,e))$.

That is, we define
\begin{equation} \label{eq:variational-family} 
q(z ; x,y,e)
=
\mathcal N\!\left(\mu_\theta(x,y,e), \diag(\sigma_\theta(x,y,e))\right).
\end{equation}

Next, we model the conditional distribution $p(y \mid z)$ with a neural predictor $f^{(1)}_\phi$, shared across environments. This defines the estimate $\phat_\phi(y \mid z)$.

For estimation of the environment conditionals, we
specify an auxiliary non-latent neural network $f^{(0)}_\varphi$ that takes $(x,e)$ as input and estimates $\phat_\varphi(y \mid x,e)$.
Similarly, we define a network $h_\psi$ that takes $x$ as input and outputs an estimate $\phat_\psi(e \mid x)$.
Using Bayes' rule, these two observed-variable conditionals define
\begin{equation}
    \phat_{\varphi,\psi}(e \mid x,y)
    \propto
    \phat_\varphi(y \mid x,e)\,
    \phat_\psi(e \mid x).
\end{equation}

Applying the marginalization identities in Equations \eqref{eq:posterior} - \eqref{eq:prior} 
to the parameterized estimators 
yields
\begin{align}
 & \qhat_{\theta,\varphi,\psi} (z ; x,y) \coloneqq \int \qhat_\theta(z ; x,y,e)\, \phat_{\varphi,\psi}(e\mid x,y)\,de, 
 \quad
 \qhat_{\theta,\varphi,\psi} (z ; x) \coloneqq \int \qhat_{\theta,\varphi}(z ; x, e)\, \phat_\psi(e\mid x)\,de, 
\end{align}
where 
\begin{equation}
    \qhat_{\theta,\varphi}(z ; x, e) \coloneqq \int q_\theta(z ; x,y,e) \, \phat_\varphi(y \mid x, e) \, dy. \label{eq:env_var}
\end{equation}

\paragraph{Optimization Objective}

Substituting these estimates into Equation \eqref{eq:pq_elbo_single}
gives our objective
\begin{equation} \label{eq:empirical}
 \!\!   \widehat{\L}(\theta,\phi,\varphi,\psi)
    \! \coloneqq \!
    \frac{1}{\vert \D \vert} \!
    \sum_{e \in \Etr} \!
    \sum_{i=1}^{n_e} \!
    \left[
    \E_{\qhat_{\theta,\varphi,\psi}(z ; x_i^e,y_i^e)}
    \left[
    \log p_\phi(y_i^e \mid z)
    \right]
    \! - \!
    \KL
    \left(
    \qhat_{\theta,\varphi,\psi}(z ; x_i^e,y_i^e)
    \, \Vert \, 
    \qhat_{\theta,\varphi,\psi}(z ; x_i^e)
    \right)
    \right].
\end{equation}
We learn all the parameters of these estimators jointly by optimizing Equation \eqref{eq:empirical}.

\begin{figure}[ht]
\vspace{-1em}
\centering
\resizebox{0.65\linewidth}{!}{%
\begin{tikzpicture}[>=stealth]

\tikzset{
  circ/.style={circle, draw=black, fill=gray!20,
               minimum size=7mm, inner sep=0pt},
  whitecirc/.style={circle, draw=black,
               minimum size=7mm, inner sep=0pt},
  sqr/.style={rectangle, draw=black,
              minimum height=7mm, minimum width=17mm,
              inner sep=2pt, align=center},
  bluesqr/.style={sqr, fill=NavyBlue!20},
  dashsqr/.style={sqr, dashed}
}

\node[circ]     (x)      at (1,   1.9) {$x$};
\node[circ]     (e)      at (1,   0.9) {$e$};
\node[circ]     (y)      at (1,  -0.1) {$y$};

\node[bluesqr]      (pyxe)   at (-1.5, 1.35) {$p_\varphi(y\mid x,e)$};
\node[bluesqr]  (pex)    at ( 3.1, 2.5) {$\phat(e\mid x)$};
\node[dashsqr]  (pexy)   at ( 6.7, 2.5) {$\phat(e\mid x,y)$};

\node[bluesqr]      (qexy)   at ( 3.1, 0.9) {$q(z ; x,y,e)$};
\node[whitecirc]     (z)      at (4.9, 0.9) {$z$};
\node[bluesqr]  (pyz)    at (6.7, 0.9) {$\hat{p}_\phi(y\mid z)$};

\node[dashsqr]  (qex)    at ( 3.1,-0.65) {$q(z ; x,e)$};

\draw[->] (x.west) -- ++(-0.5,0) |- (pyxe.east);
\draw[->] (e.west) -- ++(-0.5,0) |- (pyxe.east);

\draw[->] (x.east) -- ++(0.5,0) |- (pex.west);

\draw[->] (x.east) -- ++(0.5,0) |- (qexy.north west);
\draw[->] (e.east) -- (qexy.west);
\draw[->] (y.east) -- ++(0.5,0) |- (qexy.south west);

\draw[->] (pex.east)  --  (pexy.west);
\draw[->] (pyxe.north)  -- ++(0.0,1.5) -|  (pexy.north);


\draw[->] (qexy.east) -- (z.west);

\draw[->] (z.east) -- (pyz.west);


\draw[->] (qexy.south) --  (qex.north);
\draw[->] (pyxe.south) |-  (qex.west);
\end{tikzpicture}
}
\caption{Schematic illustration of the training computations. Observed variables are gray circles. Blue rectangles denote quantities estimated directly by neural networks. Dashed rectangles denote derived quantities, computed from other estimates.}
\label{fig:algorithm-diagram}
\end{figure}
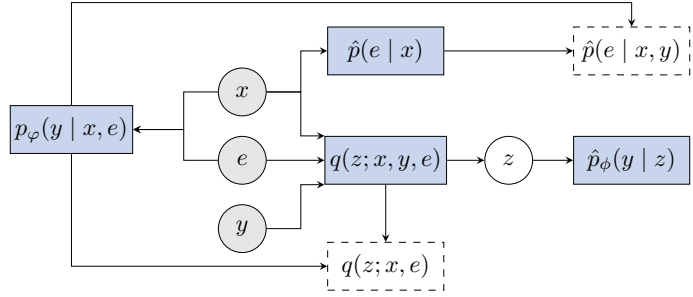

Note that $\phat_\psi(e \mid x)$ is learned implicitly: it enters the objective only through the aggregation weights in $\qhat(z ; x)$. In low-signal settings, adding a supervised environment-prediction term
$\lambda_\mathrm{env}\,
    \nicefrac{1}{\vert \D \vert}
    \sum_{e \in \Etr}
    \sum_{i=1}^{n_e}
    \log \phat_\psi(e \mid x_i^e)$
may produce better estimates as it encourages the classifier to maximize the likelihood of the environments explicitly.
We use this term only in the deliberately low-signal simulation in Section \ref{sec:param_sim}, where we set $\lambda_\mathrm{env}=1$; in all other experiments, we found no meaningful difference and set $\lambda_\mathrm{env}=0$.

Unlike reconstruction-based approaches, our objective does not require $z$ to retain all the information in $x$, but rather the information relevant for predicting the scalar outcome $y$. Accordingly, in all experiments we set the dimension of $z$ to one.

In practice, we approximate the expectation terms in Equation \eqref{eq:empirical} with a single-sample Monte Carlo estimator using standard Gaussian reparametrization. 
The resulting amortized training procedure is summarized in Algorithm \ref{alg} and is illustrated in Figure \ref{fig:algorithm-diagram}.

\begin{algorithm}[ht]
\caption{}
\label{alg}
{\renewcommand{\baselinestretch}{1.0}\selectfont
\begin{algorithmic}[1]
\Require Dataset $\D$, initial values $(\theta_0,\phi_0,\varphi_0,\psi_0)$, number of iterations $T$, batch size $b$, step size $\eta$
\For{$t=1,\dots,T$}
    \State Sample a minibatch $B=\{(x_i,y_i,e_i)\}_{i=1}^b$ from $\D$
    \For{each $(x_i,y_i,e_i)\in B$}
        \State Compute $q(z ; x_i,y_i,e_i) = g_{\theta_{t-1}}(x_i,y_i,e_i)$
        \State Sample $z_i \sim q(z ; x_i,y_i,e_i)$ by Gaussian reparametrization~\eqref{eq:variational-family}
        \State Evaluate the likelihood term $\p_{\phi_{t-1}}(y_i\mid z_i)$ using $f^{(1)}_{\phi_{t-1}}(z_i)$
        \State For all $e\in\Etr$ compute \vspace{-0.7em}
        \begin{align*}  
         & p_{\varphi_{t-1}}(y_{i}\mid x_{i},e) \text{ using } f^{(0)}_{\varphi_{t-1}}(x_{i},e)\\
         & q(z ; x_{i},e)= \textstyle \int q(z ; x_{i},y,e)\,p_{\varphi_{t-1}}(y\mid x_{i},e)\,dy \\
         & \phat_{\psi_{t-1}}(e\mid x_{i})=h_{\psi_{t-1}}(x_{i})\\
         & \phat(e\mid x_{i},y_{i})\propto\phat_{\psi_{t-1}}(e\mid x_{i})\,p_{\varphi_{t-1}}(y_{i}\mid x_{i},e)
        \end{align*}
        \State Compute
        $
        \qhat(z ; x_i,y_i)
        =
        \sum_{e\in\Etr}\phat(e\mid x_i,y_i)\,q(z ; x_i,y_i,e)
       $
        \State Compute
        $\qhat(z ; x_i)
        =
        \sum_{e\in\Etr}\phat_{\psi_{t-1}}(e\mid x_i)\,q(z ; x_i,e)$
    \EndFor
    \State Form $\Lhat_B$ as the batch approximation of Equation \eqref{eq:empirical}
    \State Update
    $(\theta_t,\phi_t,\varphi_t,\psi_t)
    =
    (\theta_{t-1},\phi_{t-1},\varphi_{t-1},\psi_{t-1})
    +
    \eta \nabla_{\theta,\phi,\varphi,\psi}\Lhat_B$
\EndFor
\State \Return $(\theta_T,\phi_T,\varphi_T,\psi_T)$  
\end{algorithmic}

}

\end{algorithm}

\subsection{Prediction}
At prediction time, we observe a new datapoint $x$, but not its environment.
Once training is complete with Algorithm \ref{alg}, prediction for a datapoint $x$, possibly from a new environment $e'$, proceeds by marginalizing over the training environments. Since $e'$ is not observed at test time, we first compute, for each $e\in\Etr$, the environment-specific conditional prior $q(z ; x,e) =\int q(z ; x,y,e)\,p_{\varphi}(y\mid x,e)\,dy$
and the environment weights $\phat_{\psi}(e\mid x)=h_{\psi}(x)$.
We then form the EB-prior 
\begin{equation}\textstyle
\qhat(z ; x)=\sum_{e\in\Etr}\phat_{\psi}(e\mid x)\,q(z ; x,e).
\end{equation}
The predictive distribution for $y$ is obtained by averaging  over this aggregated distribution,
\begin{equation} \textstyle
\phat(y\mid x)
=
\int \phat_{\phi}(y\mid z)\,\qhat(z ; x)\,dz
=
\sum_{e\in\Etr}\phat_{\psi}(e\mid x)
\int \phat_{\phi}(y\mid z)\,q(z ; x,e)\,dz. \label{eq:pred}
\end{equation}

\section{Extrapolation to New Environments} \label{sec:motivating_example}

A central question is why the EmBER model can extrapolate to observation $x$ from a new environment $e' \notin \Etr$? 
The answer is that the learned distribution $\phat_\psi(e \mid x)$ defines a prediction rule that adaptively averages over the per-environment effects:
at test time, the model evaluates the latent representation induced by each training environment, and combines these contributions according to how plausible each of them is for the observed covariate $x$.

Specifically, Equation \eqref{eq:pred} can be interpreted as expressing the new observation in terms of its similarity to the training environments: it is a weighted average of environment-specific predictions across the training environments, with weights given by $\phat_\psi(e \mid x)$. 

This averaging of environment-specific predictions over observed environments is the empirical-Bayes step: the prior $\qhat(z\mid x)$ is not prespecified, but estimated from the data.
For connections to EB $g$-modeling, see \citet{laird1978nonparametric}, especially Equation~(3.3) on self-consistent priors.

To gain intuition for the resulting extrapolation, we next present a simple example illustrating the importance of explicitly modeling $p(e \mid x)$ when environments differ in the distribution of latent~states.

\vspace{-0.3em}\paragraph{Motivating Example: Explicit Modeling of Environments} 
Let $e\sim p(e)$ be a random environment, and let $\pi_e\in[0,1]$ be its environment-specific parameter. 
Assume that $z\mid e \sim \operatorname{Bernoulli}(\pi_e)$ and 
\begin{equation}
    p(x\mid z=1)=1+\alpha(2x-1),
    \qquad
    p(x\mid z=0)=1-\alpha(2x-1),
\end{equation}
where $0<\alpha<1$ is a known parameter. The label is a noisy copy of $z$:
\begin{equation}
    p(y=1\mid z=1)=1-\rho_y,
    \qquad
    p(y=1\mid z=0)=\rho_y,
    \qquad
    0<\rho_y<1/2.
\end{equation}
As in the assumed model, $z$ depends on $e$ while $p(x\mid z)$ and $p(y\mid z)$ are stable across environments. 

Suppose we observe $m$ training environments $e_1,\dots,e_m$, with environment-specific parameters $\pi_{e_1},\dots,\pi_{e_m}$.
Even in an idealized setting where the environment-specific posteriors $p(z\vert x,e_j)$ are known, accurate prediction requires weighting training environments according to the observation $x$.

In this example, these weights are not uniform: large values of $x$ are evidence for $z=1$, which makes environments with larger $\pi_{e_j}$ more plausible. Consequently, the correct conditional environment weights $p(e \mid x)$ yield convergence to the population predictive distribution at rate $m^{-1/2}$. In contrast, uniform weights, which ignore the dependence on $x$, converge to a biased limit. 
Formal statements and proofs of these results are provided in Propositions \ref{prop:example_correct} and \ref{prop:example_incorrect} in Appendix \ref{sup:example}. Figure \ref{fig:example} illustrates this phenomenon numerically. Additional details are provided in Appendix \ref{sup:exp_example}.

\begin{figure}
    \centering
    \includegraphics[width=\linewidth]{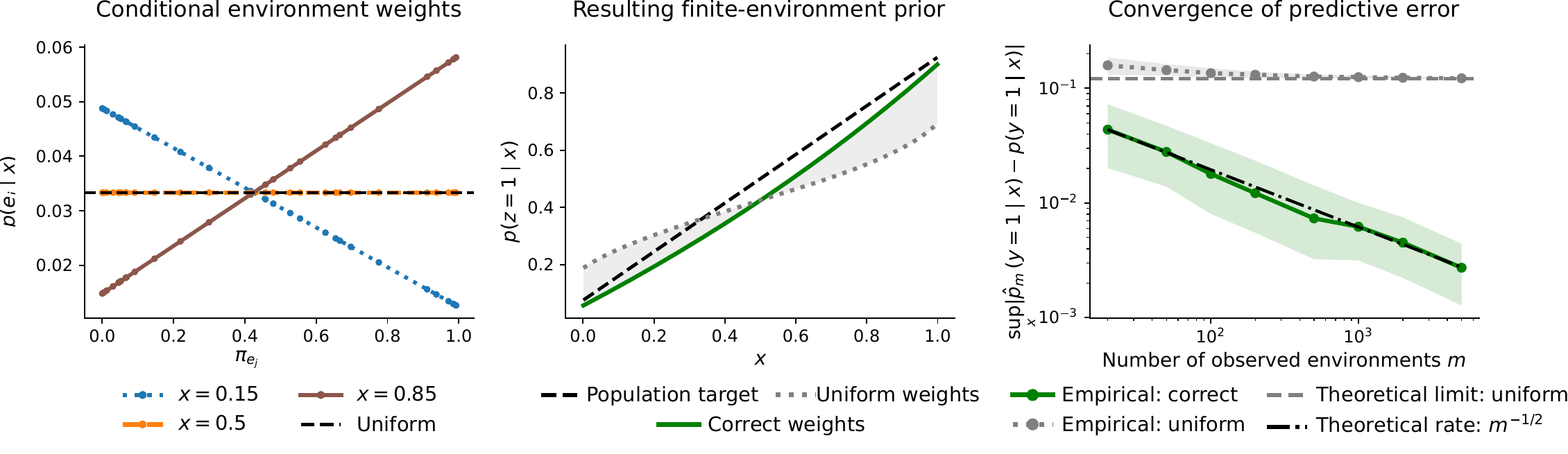}
    \caption{Illustration of the role of conditional environment weights.
Left:  weights $p(e_j\mid x)$ over observed training environments vary with $x$; larger values of $x$ favor environments with larger $\pi_{e_j}$.
Middle: these weights determine the finite-environment approximation to $p(z=1\mid x)$. Correct conditional weighting tracks the population target, whereas uniform weighting is biased.
Right: as the number of observed environments increases, the predictor based on correct weights converges at $m^{-1/2}$ rate, while the uniformly weighted predictor converges to a biased limit.}
    \label{fig:example}
\end{figure}

This example illustrates that accurate prediction requires approximation of both the environment-specific latent distributions $p(z\mid x, e)$ and the observation-specific environment-weights $p(e\mid x)$, beyond naive empirical frequencies that yield non-adaptive uniform weights.

\section{EmBER as an Environment-Balancing Method} \label{sec:balancing}
We now show that EmBER can be framed as an environment-balancing method. Specifically, 
the expected value of the lower bound in Equation \eqref{eq:pq_elbo_single} decomposes into an average of environment-specific terms, and an additional 
cross-environment discrepancy term that compares the environment-specific variational components with their global 
counterparts. The next theorem makes this precise.

\begin{theorem} \label{prop:population}
Let
\begin{equation*}
\L_e
\coloneqq
\E_{p(x,y \mid e)}
\left[
\E_{q(z ; x,y,e)}[\log p(y \mid z)]
-
\KL(q(z ; x,y,e)\, \Vert \, q(z ; x,e))
\right],
\end{equation*}
and denote $\L \coloneqq \E_{p(x,y)}[\L(x,y)]$.
Then, $\L
 = 
\E_{p(e)}
\left[
\L_e
\right]
+
R(q)$
where
\begin{align*}
R(q) \coloneqq
\E_{p(x,y)}
\Big[
\E_{p(e\vert x,y)}[\KL(q(z ; x,y,e)\, \Vert \, q(z ; x,y))]
\Big]  - 
\E_{p(x)}
\Big[
\E_{p(e\vert x)}[\KL(q(z ; x,e)\, \Vert \, q(z ; x))]
\Big].
\end{align*}
\end{theorem}
The proof is provided in Appendix \ref{proof:population}. 
Theorem \ref{prop:population} shows that EmBER induces an environment-balancing term that follows directly from the latent-variable model. 
This balancing term $R(q)$ is composed from two discrepancy terms. The first compares the environment-specific posteriors $q(z ; x,y,e)$ with the global posterior $q(z ; x,y)$. The second compares the corresponding environment-specific priors $q(z ; x,e)$ with the global prior $q(z ; x)$. The balancing term is large when posterior discrepancy across environments exceeds the discrepancy already present in the corresponding priors.

The decomposition in Theorem \ref{prop:population} also shows why naive estimation of  $\phat(e \mid x,y)$ and $\phat(e \mid x)$ directly from empirical frequencies in the training data is inadequate.
In continuous settings, each observed pair $(x_i,y_i)$ typically appears in only one environment $e_i$. Consequently, the empirical conditionals collapse to point masses yielding $\qhat(z ; x_i,y_i)=q(z ; x_i,y_i,e_i)$ and $\qhat(z ; x_i)=q(z ; x_i,e_i)$.
As a result, both cross-environment discrepancy terms in $R(q)$ vanish, and the objective reduces to an average of environment-specific objectives. 
In this collapsed regime, the aggregation becomes trivial and the environment weights $p(e \mid x)$ and $p(e \mid x,y)$  are not learned or used during training.

However, this collapsed objective is incompatible with the intended test-time use. At test time, the environment label is unobserved, so prediction must aggregate the environment-specific quantities $q(z ; x,e)$ using the observation-specific weights $\phat(e \mid x)$. If training uses empirical point-mass weights, each observation is assigned only to its observed environment, so the model is never forced to learn how environment-specific components should be combined for a new $x$. Preserving the global objective therefore requires modeling the environment weights explicitly: $\hat p(e\mid x,y)$ determines how environment-specific posteriors are combined during training, while $\hat p(e\mid x)$ determines how environment-specific priors are combined during both training and prediction.

\section{Empirical Results}

We evaluate EmBER on synthetic and real-world multi-environment classification tasks. The synthetic experiments are designed to match the assumed structure. The real-world experiments test the same principle in settings with naturally occurring environments.
We evaluate predictive performance using negative log-likelihood (NLL) and classification accuracy.

We compare EmBER with empirical risk minimization (ERM) and with environment-balancing methods: IRM \citep{IRM}, VREx \citep{krueger2021out}, and Fishr \citep{rame2022fishr} (see details in \S \ref{sec:related}). 
Within each experiment, all methods use the same train-test environment split and model. We test each method's ability to form accurate predictions on held-out environments. 

Across all 5 experiments, EmBER is the strongest method overall. In the parametric simulation, it achieves the lowest NLL and highest accuracy, with accuracy nearly matching the Bayes-optimal predictor, and remains best across the sensitivity and ablation studies. The same pattern holds in the colored MNIST simulation. On the real-world tasks, EmBER again performs best, with substantial improvements over all competitors in quasar-star and over all competitors except VREx in microbiome. In the sepsis task, EmBER achieves the highest AUPRC and substantial AUROC gains over all methods except IRM. An additional method-agreement analysis shows that EmBER's mistakes are not concentrated on examples that are consistently easy for competing methods.

For each experiment, we perform one-sided $t$-tests comparing EmBER with the competing methods and apply the Benjamini--Yekutieli correction for multiple comparisons. 
Additional details for all experiments are provided in Appendix \ref{sup:exp}, and additional results in Appendix \ref{sup:results}.
Code is available at \url{https://github.com/mattyshen/EmBER}.

\subsection{Parametric Simulation} \label{sec:param_sim}

We first study a controlled setting. For each environment $j$, we assign a branch index $b_j$ and draw an environment-specific branch center $m_j=\Delta b_j+\phi_0+\delta_j$ with $\delta_j\sim\mathcal{N}(0,\sigma_\delta^2)$.
Then, for observation $i$ in environment $j$, we sample $z_{ij}\mid e_j \sim \mathcal{N}(m_j,\sigma_z^2)$.
The observed covariates are generated as
\begin{equation}
x_{ij,1}
=
\cos(z_{ij})+\varepsilon_{ij,1},
\qquad
x_{ij,2}
=
\sin(z_{ij})+\varepsilon_{ij,2},
\qquad
x_{ij,3}
=
\rho z_{ij}+\varepsilon_{ij,3},
\end{equation}
and $x_{ij,k}=\varepsilon_{ij,k}$ for $k=4,\ldots,d$. The first two coordinates reveal the phase of $z_{ij}$ but not its branch, while the third coordinate provides a noisy branch cue. The target is generated by the stable latent mechanism
$
y_{ij}\mid z_{ij}
\sim
\mathrm{Bernoulli}
\left(
\sigma\left(\beta\sin(z_{ij}/2)\right)
\right).
$

\begin{table}[ht]
    \centering
    \small
    \begin{tabular}{lccc}
    \toprule
    Method & NLL $\downarrow$ & Accuracy $\uparrow$ & PV (NLL) \\
    \midrule
    Bayes-x & 0.233 $\pm$ 0.004 & 0.934 $\pm$ 0.002 & -- \\
    \midrule
    EmBER & \textbf{0.319 $\pm$ 0.069} & \textbf{0.928 $\pm$ 0.003} & -- \\
    ERM & 0.652 $\pm$ 0.005 & 0.577 $\pm$ 0.009 & 0.015 \\
    IRM & 0.654 $\pm$ 0.006 & 0.573 $\pm$ 0.012 & 0.015 \\
    V-REx & 0.640 $\pm$ 0.026 & 0.572 $\pm$ 0.016 & 0.015 \\
    Fishr & 0.652 $\pm$ 0.006 & 0.582 $\pm$ 0.011 & 0.015 \\
    \bottomrule
    \end{tabular}
    \vspace{0.5em}
    \caption{Parametric simulation results. EmBER performs best among all learned methods and nearly matches optimal-Bayes accuracy.}
    \label{tab:parametric}
\end{table}

\textbf{Sensitivity and ablation analysis} \hspace{0.5em}
We vary branch-cue strength $\rho$, the number of training environments $m_{\mathrm{train}}$, number of observations per environment $n$, and data dimension $d$.
We compare ERM, IRM, VREx, Fishr, EmBER, and EmBER variants that replace learned components with oracle values. 

\textbf{Results} \hspace{0.5em} Table \ref{tab:parametric} shows results for the baseline parameters. As an oracle reference, we report the Bayes-optimal predictor (Bayes-x).
EmBER yields the highest accuracy and the lowest NLL among all non-oracle methods.
The improvement is statistically significant relative to all methods with $p=0.015$.  
Across 8 sensitivity studies and 6 ablation studies, EmBER remains best overall among methods that we implemented.
Detailed results and analysis of these studies are reported in Appendix \ref{sup:ablation}. 

\subsection{Colored MNIST Simulation}

To illustrate the setting targeted by our method, we construct a variant of the colored MNIST experiment \citep{IRM} in which the environment affects the data only through the distribution of a latent prototype variable $z \in \{1,2,3,4\}$. We generate data according to
\begin{equation*}
e \sim p(e), \qquad
z \sim p(z\vert e), \qquad
y \sim p(y\vert z), \qquad
d \sim p(d\vert z), \qquad
c \sim p(c\vert z), \qquad
x \sim p(x\vert d,c),
\end{equation*}

\begin{table}[ht]
    \centering
    \small
    \begin{tabular}{lccc}
    \toprule
    Method & NLL $\downarrow$ & Accuracy $\uparrow$ & PV (NLL) \\
    \midrule
    EmBER & \textbf{0.519 $\pm$ 0.077} & \textbf{0.787 $\pm$ 0.071} & -- \\
    ERM & 0.863 $\pm$ 0.090 & 0.650 $\pm$ 0.021 & 0.004 \\
    IRM & 0.572 $\pm$ 0.007 & 0.747 $\pm$ 0.017 & 0.203 \\
    V-REx & 0.605 $\pm$ 0.019 & 0.684 $\pm$ 0.011 & 0.140 \\
    Fishr & 0.917 $\pm$ 0.061 & 0.639 $\pm$ 0.020 & 0.004 \\
    \bottomrule
    \end{tabular}
    \vspace{0.5em}
    \caption{Colored MNIST simulation results. EmBER performs best among all methods.}
    \label{tab:cmnist}
\end{table}

where $y \in \{0,1\}$ is the binary label, $d \in \{0,\dots,9\}$ is the digit, $c \in \{\mathrm{red},\mathrm{green}\}$ is the color, and $x$ is the corresponding colored image.
The environment affects only the latent prototype mixture $p(z\vert e)$; it does not directly alter the label mechanism or image rendering.
Given $z$, the digit identity is sampled from a prototype-specific distribution over digits, so $d$ carries no environment-specific information beyond $z$.

The four prototypes have stable color and label tendencies: $z\in\{1,2\}$ is mostly associated with label $0$, while $z\in\{3,4\}$ is mostly associated with label $1$; $z\in\{1,3\}$ is mostly red, while $z\in\{2,4\}$ is mostly green.
Each prototype also has a distinct, overlapping distribution over digit identities $p(d\vert z)$.
Thus, training environments differ only through the prototype mixture $p(z\vert e)$.
We generate six training environments and one test environment whose prototype mixture reverses the induced color-label association.
Additional details are provided in Appendix \ref{sup:mnist_details}.
 
\textbf{Results} \hspace{0.5em} Table \ref{tab:cmnist} shows that EmBER yields the lowest NLL and highest accuracy among all methods.

\subsection{Quasar-Star Classification}
 %
%
\begin{wrapfigure}[13]{r}{0.49\textwidth}
    \vspace{-3.5em}
    \centering
    \includegraphics[width=\linewidth]{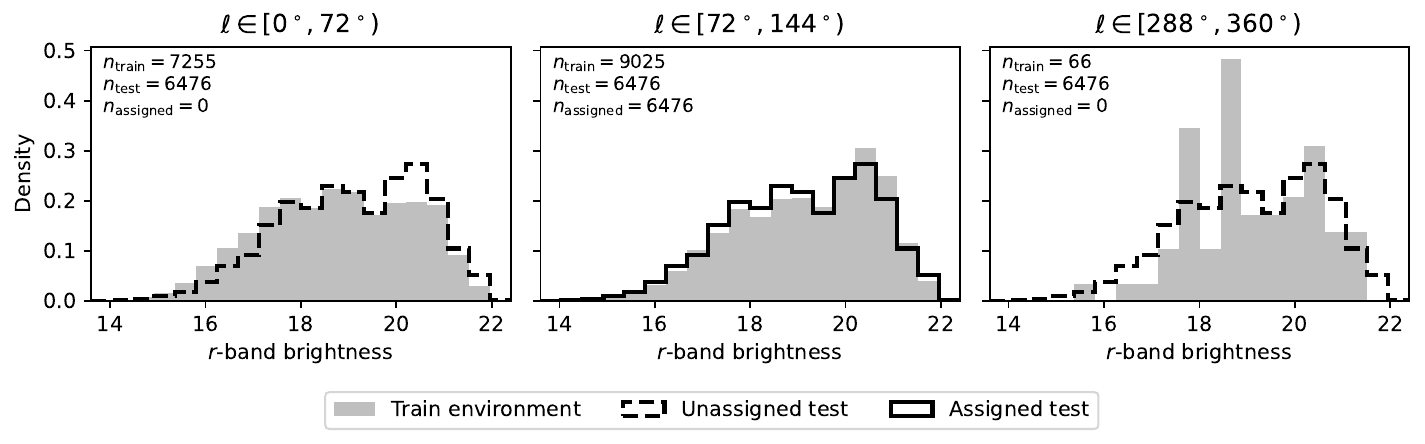}
    \caption{Quasar-star $r$-band brightness distributions by training environment.
    Gray histograms show the distribution in each training environment indexed by $\ell$.
    For each test point, we assign the point to
    $\arg\max_e \hat p(e \mid x_{\mathrm{test}})$.
    All test points are assigned to training environment
    $\ell \in [72^\circ,144^\circ)$, whose distribution most closely aligns
    with the test distribution; unassigned environments show poorer agreement
    with the test distribution.}
    \label{fig:dists_partial}
\end{wrapfigure}
We use the publicly available Sloan Digital Sky Survey dataset \citep{ahumada202016th}. The covariates $x$ are photometric color features, the label $y$ is quasar status, and the environment $e$ is the sky-position group defined by binned absolute Galactic latitude $\vert b\vert$.

\textbf{Results} \hspace{0.5em}
Table \ref{tab:stars} shows that EmBER yields the highest accuracy and the lowest NLL among all methods.
The improvement is statistically significant relative to all methods with $p<0.001$.

Figure \ref{fig:SDSS_agreement} examines the subset of test points on which EmBER disagrees with competing methods and reports how often the competing methods classify these points correctly. No competing method performs consistently well on this subset, indicating that EmBER's errors are not concentrated on examples that are systematically easy for the other methods.

We additionally analyze how the environments are assigned by the learned model. For all test points, EmBER assigns the largest estimated weight $\hat p(e\mid x)$ to the same training environment. Figure \ref{fig:dists_partial} compares the train and test distributions for the most variable feature; analogous comparisons for additional features are reported in Appendix Figure \ref{fig:dists_full}. The results show that EmBER assigns highest weight to the environment whose distribution is most similar to the test distribution.

\begin{table}[ht]
\centering
\begin{tabular}{lccc}
\hline
Method & NLL $\downarrow$ & Accuracy $\uparrow$ & PV (NLL) \\
\hline
EmBER & 0.221 $\pm$ 0.019 & 0.921 $\pm$ 0.007 & -- \\
ERM & 0.390 $\pm$ 0.001 & 0.897 $\pm$ 0.001 & $< 0.001$ \\
Fishr & 0.384 $\pm$ 0.004 & 0.894 $\pm$ 0.001 & $< 0.001$ \\
IRM & 0.383 $\pm$ 0.009 & 0.891 $\pm$ 0.003 & $< 0.001$ \\
VREx & 0.382 $\pm$ 0.009 & 0.892 $\pm$ 0.003 & $< 0.001$ \\
\hline
\end{tabular}
\vspace{0.5em}
\caption{Quasar-star classification results. EmBER performs best among all methods.}
\label{tab:stars}
\end{table}
\vspace{-1.5em}
\begin{table}[ht]
\centering
\begin{tabular}{lccc}
\hline
Method & NLL $\downarrow$ & Accuracy $\uparrow$ & PV (NLL) \\
\hline
EmBER & 0.850 $\pm$ 0.112 & 0.591 $\pm$ 0.027 & -- \\
ERM & 1.685 $\pm$ 0.136 & 0.587 $\pm$ 0.016 & 0.002 \\
Fishr & 1.449 $\pm$ 0.226 & 0.582 $\pm$ 0.021 & 0.008 \\
IRM & 1.548 $\pm$ 0.224 & 0.588 $\pm$ 0.011 & 0.007 \\
VREx & 0.903 $\pm$ 0.087 & 0.571 $\pm$ 0.014 & 0.543 \\
\hline
\end{tabular}
\vspace{0.5em}
\caption{Microbiome classification results. EmBER performs best among all methods.}
\label{tab:microbiome}
\end{table}

\subsection{Microbiome Classification}
We use the publicly available MicrobiomeHD datasets \citep{duvallet2017meta} for a colorectal-cancer microbiome classification task. Here, covariates $x$ are genus-level abundance profiles, $y$ is colorectal-cancer status, and environments $e$ are study cohorts. Appendix~\ref{sup:microbiom_details} provides additional details. 

\textbf{Results} \hspace{0.5em}
Table \ref{tab:microbiome} shows that EmBER achieves the highest accuracy and the lowest NLL among all methods. The improvement is statistically significant with $p<0.01$ relative to all methods except VREx.
The results in Figure \ref{fig:ZENODO_agreement} further show that here as well, no competing method performs consistently well on the subset of test points on which EmBER is incorrect.

\subsection{ICU Sepsis Prediction}
We evaluate all methods on a multi-environment sepsis prediction task derived from the PhysioNet 2019 Challenge data \citep{reyna2020early}. Each example corresponds to one ICU stay. Environments are defined by the public data source and the recorded ICU unit, serving as coarse hospital-systems. Covariates $x$ are demographic variables and summaries of vital signs and laboratory measurements from the first 24 hours, labels $y$ indicate whether a patient becomes septic later in the stay, and environments $e$ are hospital-system groups, defined by the public data source and recorded ICU unit.
For each ICU stay, we use measurements from the first 24 hours and predict whether the patient becomes septic later in the stay. Patients whose sepsis label becomes positive during the first 24 hours are excluded. The covariates consist of demographic variables and summaries of early vital signs and laboratory measurements. 

\begin{table}[ht]
    \centering
    \small
    \begin{tabular}{lccc}
    \toprule
    Method & AUROC $\uparrow$ & AUPRC $\uparrow$ & PV (AUROC) \\
    \midrule
    EmBER & \textbf{0.740 $\pm$ 0.008} & \textbf{0.179 $\pm$ 0.008} & -- \\
    ERM & 0.692 $\pm$ 0.017 & 0.153 $\pm$ 0.013 & 0.003 \\
    Fishr & 0.697 $\pm$ 0.014 & 0.156 $\pm$ 0.008 & 0.002 \\
    IRM & 0.724 $\pm$ 0.012 & 0.169 $\pm$ 0.012 & 0.112 \\
    VREx & 0.697 $\pm$ 0.012 & 0.154 $\pm$ 0.011 & 0.001 \\
    \bottomrule
    \end{tabular}
    \vspace{0.5em}
    \caption{Sepsis prediction results. EmBER performs best among all methods.}
    \vspace{-1.5em}
    \label{tab:sepsis}
\end{table}
%
We train on five source-unit environments and evaluate on a held-out source-unit environment. Since sepsis is rare, we evaluate predictions using AUROC and AUPRC.

\textbf{Results} \hspace{0.5em}  Table \ref{tab:sepsis} shows that EmBER achieves the highest AUROC and AUPRC among all methods. The improvement is statistically significant with $p<0.004$ relative to all methods except IRM. 

\vspace{-0.5em}
\section{Discussion} \label{sec:discussion}
We studied multi-environment prediction problems where environments change the distributions of the latent variables, while the covariate and label mechanisms remain stable. Our goal was to recover the marginal Bayes rule $p(y\mid x)$ by averaging over latent and environment-specific variables.

Our main contribution is an empirical-Bayes variational method for carrying out this marginalization. Our analysis shows that the corresponding objective induces a specific balancing term: the relevant penalty is a discrepancy term that is implied by the latent-variable model itself. The analysis also shows that generalization to new environments requires explicit estimation of the environment effect through $\phat(e\mid x)$ and $\phat(e\mid x,y)$. The empirical results show that our method produces substantially better predictive performance compared to other methods.

Our model assumes that environments affect the latent distribution $p(z\mid e)$, while $p(x\mid z)$ and $p(y\mid z)$ remain stable. Since $z$ is unobserved, these assumptions cannot be directly tested from observational data alone. However, their predictive implications can be assessed on held-out environments: strong predictive performance supports the adequacy of these assumptions for generalization.

Unlike other variational approaches, our method does not require prior specification, instead it uses the data to inform the prior. However, computing the aggregate distributions $\qhat(z\mid x,y)$ and $\qhat(z\mid x)$ increases the computational cost. 
A promising direction for future work is approximating these aggregates with a VampPrior-inspired mixture \citep{tomczak2018vae} built from a smaller number of learned pseudo-observations.

\clearpage 
\newpage
\setcitestyle{numbers}
\bibliographystyle{plainnat}
\bibliography{references}
\newpage

\appendix
\setcitestyle{authoryear}

\section{Broader Related Work} \label{sup:related}

In this section, we discuss broader connections beyond the most directly related literature in \S \ref{sec:related}.

\paragraph{Causal representation learning.}
Beyond environment-balancing criteria-based methods, causal representation learning has been used to improve generalization across environments by identifying latent causal factors underlying high-dimensional observations \citep{scholkopf2021toward}. Causally motivated approaches have improved generalization under shifts involving previously unseen combinations of attributes \citep{atzmon2020causal}, changes in experimental conditions \citep{greenfeld2020robust}, and shifts in context \citep{mao2021generative}. However, identifying causal representations often requires additional information or structural assumptions, such as paired observations before and after interventions \citep{brehmer2022weakly} or interventions on latent variables \citep{ahuja2023interventional}.

Our approach similarly models observed variables through latent factors, but does not aim to identify causal variables, causal relationships, or intervention effects. Instead, we assume that environments change $p(z\mid e)$ while $p(x\mid z)$ and $p(y\mid z)$ remain stable. EmBER then learns to marginalize over environment-specific latent distributions to predict in new environments, without requiring identification of the latent variables as causal factors.

\paragraph{Mixed-effects models.}
Mixed-effects models capture group-specific variation through random effects while maintaining shared components \citep{McCullaghNelder1989GLM,demidenko2013mixed}. Estimating the parameters of the random-effects distribution by maximizing the marginal likelihood can be viewed as an empirical-Bayes procedure, in which the prior is learned from the data. Neural mixed-effects models extend this framework to nonlinear relationships \citep{Simchoni_Rosset1,Simchoni_Rosset2,MeNets,DeepGLMM, slavutsky2026neural}.

In EmBER, environments similarly index different distributions $p(z\mid e)$. The key distinction is that EmBER does not prescribe a parametric relationship between the observed covariates and the environment. Instead, it flexibly estimates $p(e\mid x)$ using a neural classifier and uses these observation-specific probabilities to aggregate the environment-specific latent distributions when predicting in new environments.

\paragraph{Worst-case robustness.}
Distributionally robust optimization (DRO) addresses distribution shifts by minimizing the worst-case expected loss over a set of plausible distributions \citep{DRO,duchi2021statistics,duchi2021learning,wei2023distributionally}. Group DRO extends this approach to predefined groups, seeking robustness to changes in their relative prevalence \citep{sagawa2019distributionally,piratla2021focus}. Related approaches address such shifts through reweighting high-loss observations \citep{liu2021just} or balancing groups through subsampling \citep{idrissi2022simple}. More recently, \citet{felekis2026distributionally} developed a DRO framework for learning causal abstractions that are robust to shifts in exogenous noise distributions.

Unlike these worst-case approaches, EmBER targets the marginal Bayes predictor under a probabilistic model of environment variation. Rather than optimizing against an adversarial distribution, it targets best possible prediction on average over environments sampled from the same distribution.

\section{Formal Analysis of the Motivating Example} \label{sup:example}

Here we provide the formal statements for the motivating example in Section \ref{sec:motivating_example}. 

Recall the setting in Section \ref{sec:motivating_example}:  $e\sim p(e)$, where $\pi_e\in[0,1]$ are environment-specific parameters. 
We assume that $z\mid e \sim \operatorname{Bernoulli}(\pi_e)$ and 
\begin{equation}
    p(x\mid z=1)=1+\alpha(2x-1),
    \qquad
    p(x\mid z=0)=1-\alpha(2x-1),
\end{equation}
where $0<\alpha<1$ is a known parameter, and that the label is a noisy copy of $z$:
\begin{equation}
    p(y=1\mid z=1)=1-\rho_y,
    \qquad
    p(y=1\mid z=0)=\rho_y,
    \qquad
    0<\rho_y<1/2.
\end{equation}

Suppose we observe $m$ training environments $e_1,\dots,e_m$, with environment-specific parameters $\pi_{e_1},\dots,\pi_{e_m}$.
Let $\mu \coloneqq \E_{p(e)}[\pi_e]$ and
 $\hat\mu_m \coloneqq \frac{1}{m}\sum_{j=1}^m \pi_{e_j}$.
Then, the global prior and its estimator (obtained by averaging over the $m$ observed environments) are
\begin{equation*}
p(z\!=\!1\vert x)
    \!=\!
    \frac{
        \mu p(x\vert z\!=\!1)
    }{
        \mu p(x\vert z\!=\!1)\!+\!(1\!-\!\mu)p(x\vert z\!=\!0)
    },
    \quad
    \hat p_m(z\!=\!1\vert x)
    \!=\!
    \frac{
        \hat\mu_m p(x\vert z\!=\!1)
    }{
        \hat\mu_m p(x\vert z\!=\!1)\!+\!(1\!-\!\hat\mu_m)p(x\vert z\!=\!0)
    }.
\end{equation*}

Since $p(y\mid z)$ is known, the predictive distribution and its finite-environment estimator are 
\begin{equation}
     p(y=1\mid x)
    =
    \rho_y+(1-2\rho_y)p(z=1\mid x),
    \quad
    \hat p_m(y=1\mid x)
    =
    \rho_y+(1-2\rho_y)\hat p_m(z=1\mid x).
\end{equation}

\paragraph{Correct conditional environment weights.}
We first analyze the estimator $\hat p_m(y=1\mid x)$ obtained by averaging over the observed environments using the correct conditional environment weights. In this case, the only source of error is the error from replacing the population average $\mu=\E_{p(e)}[\pi_e]$ by its empirical version $\hat\mu_m$. As a result, the finite-environment Bayes estimator $\hat p_m(y = 1 \mid x)$ converges to the population predictive distribution $p(y = 1 \mid x)$ at a rate of $m^{-1/2}$. This is formalized in the next proposition.

\begin{proposition} \label{prop:example_correct}
For every $\epsilon>0$ and $\delta\in(0,1)$, to guarantee
\begin{equation}
    \Pr
    \left(
    \sup_{x\in[0,1]}
    \vert \hat p_m(y=1\mid x)-p(y=1\mid x)\vert
    \le
    \epsilon
    \right)
    \ge
    1-\delta,
\end{equation}
it is sufficient that
\begin{equation}
    m
    \ge
    \frac{
        (1-2\rho_y)^2\kappa^2
    }{
        2\epsilon^2
    }
    \log\frac{2}{\delta}, 
    \qquad \text{with} \quad 
    \kappa \coloneqq \frac{1+\alpha}{1-\alpha}.
\end{equation}
\end{proposition}

The proof is provided in Appendix \ref{proof:example_correct}.

The same conclusion, however, does not hold if we replace the conditional environment weights by uniform weights over the observed environments.
We analyze this case next.

\paragraph{Uniform environment weights.}

Consider now the approximation obtained by replacing the conditional environment weights with uniform weights over the observed training environments. This changes the estimator from a plug-in estimate of $f_x(\mu)$ to an empirical average of $f_x(\pi_{e_j})$. Because $f_x$ is nonlinear whenever $x\neq 1/2$, these two quantities have different population limits in general. The next proposition formalizes this bias.

\begin{proposition} \label{prop:example_incorrect}
Assume that $\operatorname{Var}_{p(e)}(\pi_e)>0$. Then, for every fixed $x\in[0,1]$ such that $x\neq 1/2$,
\begin{equation}
    \tilde p_m(y=1\vert x)
    \longrightarrow
    \rho_y+(1-2\rho_y) \,
    \E_{p(e)}
    \left[
    \frac{
        \pi_e \,p(x\vert z=1)
    }{
        \pi_e \, p(x\vert z=1)+(1-\pi_e)\, p(x\vert z=0)
    }
    \right]
\end{equation}
almost surely, and the limiting value is different from $p(y=1\vert x)$. Consequently, for every
\begin{equation}
    0<\epsilon<
    \left\vert
    \rho_y+(1-2\rho_y)
    \E_{p(e)}
    \left[
    \frac{
        \pi_e \, p(x\vert z=1)
    }{
        \pi_e \, p(x\vert z=1)+(1-\pi_e)\,p(x\vert z=0)
    }
    \right]
    -
    p(y=1\vert x)
    \right\vert,
\end{equation}
we have that as $m \to \infty $, 
\begin{equation}
    \Pr
    \left(
    \vert
    \tilde p_m(y=1\vert x)-p(y=1\vert x)
    \vert
    >
    \epsilon
    \right)
    \longrightarrow 1.
\end{equation}
\end{proposition}

The proof is provided in Appendix \ref{proof:example_incorrect}.

\section{Proofs} \label{sup:proofs}

\subsection{Proof of Theorem \ref{prop:population}} \label{proof:population}

\begin{proof}
First, note that
\begin{equation} \label{eq:like_xy}
\E_{q(z\vert x,y)}[\log p(y\vert z)]
=
\E_{p(e\vert x,y)}\E_{q(z\vert x,y,e)}\left[\log p(y\vert z)\right].
\end{equation}
Second, by definition
\begin{equation}
\KL(q(z\vert x,y,e) \, \Vert \, q(z\vert x))
=
\E_{q(z\vert x,y,e)}
\left[
\log \frac{q(z\vert x,y,e)}{q(z\vert x)}
\right].
\end{equation}
Taking expectation with respect to $p(e\vert x,y)$ and inserting $q(z\vert x,y)$, we get
\begin{align} \label{eq:kl_xy} 
\E_{p(e\vert x,y)}[\KL(q(z\vert x,y,e) \, \Vert \, q(z\vert x))]
= &
\E_{p(e\vert x,y)}
\E_{q(z\vert x,y,e)}
\left[
\log \frac{q(z\vert x,y,e)}{q(z\vert x,y)}
\right]\\
& +
\E_{q(z\vert x,y)} \notag
\left[
\log \frac{q(z\vert x,y)}{q(z\vert x)}
\right] \\ 
= &
\E_{p(e\vert x,y)}[\KL(q(z\vert x,y,e) \, \Vert \, q(z\vert x,y))]
+
\KL(q(z\vert x,y) \, \Vert \, q(z\vert x)). \notag
\end{align}
Rearranging the last equation, we have
\begin{align} 
\KL(q(z\vert x,y) \, \Vert \, q(z\vert x)) \label{eq:exp_kl_xy}
= & 
\E_{p(e\vert x,y)}[\KL(q(z\vert x,y,e) \, \Vert \, q(z\vert x))] \\
& -
\E_{p(e\vert x,y)}[\KL(q(z\vert x,y,e) \, \Vert \, q(z\vert x,y))]. \notag
\end{align}

Substituting Equations \eqref{eq:like_xy} and \eqref{eq:exp_kl_xy} into $\L(x,y)$ we get
\begin{align}
\L(x,y)
= &
\E_{p(e\vert x,y)}
\Big[
\E_{q(z\vert x,y,e)}[\log p(y\vert z)]\\
&-
\KL(q(z\vert x,y,e) \, \Vert \, q(z\vert x))
\Big]
+
\E_{p(e\vert x,y)}[\KL(q(z\vert x,y,e) \, \Vert \, q(z\vert x,y))]. \notag
\end{align}

Averaging over $p(x,y)$ and using $p(x,y)p(e\vert x,y)=p(e)p(x,y\vert e)$ yields
\begin{align} \notag
\L
= &
\E_{p(x,y)}
\Big[
\E_{p(e\vert x,y)}
\Big[
\E_{q(z\vert x,y,e)}[\log p(y\vert z)]
-
\KL(q(z\vert x,y,e) \, \Vert \, q(z\vert x))
\Big]
\Big] \\
& +
\E_{p(x,y)}
\Big[
\E_{p(e\vert x,y)}[\KL(q(z\vert x,y,e) \, \Vert \, q(z\vert x,y))]
\Big] \\
= &
\E_{p(e)}
\Big[
\E_{p(x,y\vert e)}
\Big[
\E_{q(z\vert x,y,e)}[\log p(y\vert z)] \\
& -
\KL(q(z\vert x,y,e) \, \Vert \, q(z\vert x))
\Big]
\Big]
+
\E_{p(x,y)}
\Big[
\E_{p(e\vert x,y)}[\KL(q(z\vert x,y,e) \, \Vert \, q(z\vert x,y))]
\Big]. \label{eq:L_xy}
\end{align}

Now, for each fixed environment $e$, by adding and subtracting $q(z\vert x,e)$, we get
\begin{align} \label{eq:kl_0}
\E_{p(x,y\vert e)}[\KL(q(z\vert x,y,e) \, \Vert \, q(z\vert x))]
= &
\E_{p(x,y\vert e)}
\E_{q(z\vert x,y,e)}
\left[
\log \frac{q(z\vert x,y,e)}{q(z\vert x)}
\right] \\
= &
\E_{p(x,y\vert e)}
\E_{q(z\vert x,y,e)}
\left[
\log \frac{q(z\vert x,y,e)}{q(z\vert x,e)}
\right]\\
& +
\E_{p(x,y\vert e)}
\E_{q(z\vert x,y,e)}
\left[
\log \frac{q(z\vert x,e)}{q(z\vert x)}
\right]. \notag
\end{align}
Note that the first term is
\begin{equation} \label{eq:first}
\E_{p(x,y\vert e)}[\KL(q(z\vert x,y,e) \, \Vert \, q(z\vert x,e))].
\end{equation}
For the second term, since
\begin{equation}
q(z\vert x,e)=\E_{p(y\vert x,e)}[q(z\vert x,y,e)],
\end{equation}
we have
\begin{align}
&\E_{p(x,y\vert e)}
\E_{q(z\vert x,y,e)}
\left[
\log \frac{q(z\vert x,e)}{q(z\vert x)}
\right] \\
&=
\E_{p(x\vert e)}
\E_{p(y\vert x,e)}
\E_{q(z\vert x,y,e)}
\left[
\log \frac{q(z\vert x,e)}{q(z\vert x)}
\right] \\
&=
\E_{p(x\vert e)}
\E_{q(z\vert x,e)}
\left[
\log \frac{q(z\vert x,e)}{q(z\vert x)}
\right] \\
&=
\E_{p(x\vert e)}[\KL(q(z\vert x,e) \, \Vert \, q(z\vert x))]. \label{eq:second}
\end{align}
Substituting Equations \eqref{eq:first} and \eqref{eq:second}  back into Equation \eqref{eq:kl_0} yields
\begin{align}
\E_{p(x,y\vert e)}[\KL(q(z\vert x,y,e) \, \Vert \, q(z\vert x))]
= & 
\E_{p(x,y\vert e)}[\KL(q(z\vert x,y,e) \, \Vert \, q(z\vert x,e))]\\
& +
\E_{p(x\vert e)}[\KL(q(z\vert x,e) \, \Vert \, q(z\vert x))], \notag
\end{align}
and substituting into Equation \eqref{eq:L_xy},
\begin{equation}
\begin{aligned}
\L
&=
\E_{p(e)}
\Bigg[
\E_{p(x,y\vert e)}
\Big[
\E_{q(z\vert x,y,e)}[\log p(y\vert z)]
-
\KL(q(z\vert x,y,e) \, \Vert \, q(z\vert x,e))
\Big]
-
\E_{p(x\vert e)}[\KL(q(z\vert x,e) \, \Vert \, q(z\vert x))]
\Bigg] \\
&\quad
+
\E_{p(x,y)}
\Big[
\E_{p(e\vert x,y)}[\KL(q(z\vert x,y,e) \, \Vert \, q(z\vert x,y))]
\Big].
\end{aligned}
\end{equation}

Therefore, we have
\begin{align}
\L
=
&\E_{p(e)}\Big[\L_e\Big]
+
\E_{p(x,y)}
\Big[
\E_{p(e\vert x,y)}[\KL(q(z\vert x,y,e) \, \Vert \, q(z\vert x,y))]
\Big] \\
&-
\E_{p(e)}
\Big[
\E_{p(x\vert e)}[\KL(q(z\vert x,e) \, \Vert \, q(z\vert x))]
\Big]. \notag
\end{align}
Using $p(x)p(e\vert x)=p(e)p(x\vert e)$ in the last term gives
\begin{align}
\L
=
&\E_{p(e)}\Big[\L_e\Big]
+
\E_{p(x,y)}
\Big[
\E_{p(e\vert x,y)}[\KL(q(z\vert x,y,e) \, \Vert \, q(z\vert x,y))]
\Big] \\
&-
\E_{p(x)}
\Big[
\E_{p(e\vert x)}[\KL(q(z\vert x,e) \, \Vert \, q(z\vert x))]
\Big]. \notag
\end{align}
\end{proof}

\subsection{Proof of Proposition \ref{prop:example_correct}} \label{proof:example_correct}

\begin{proof} 
For fixed $x$, denote
\begin{equation}
    f_x(a)
    \coloneqq
    \frac{
        a p(x\mid z=1)
    }{
        a p(x\mid z=1)+(1-a)p(x\mid z=0)
    }.
\end{equation}
Then,
\begin{equation}
    p(z=1\mid x)=f_x(\mu),
    \qquad
    \hat p_m(z=1\mid x)=f_x(\hat\mu_m).
\end{equation}
We have that
\begin{equation}
    f_x'(a)
    =
    \frac{
        p(x\mid z=1)p(x\mid z=0)
    }{
        \left(
        a p(x\mid z=1)+(1-a)p(x\mid z=0)
        \right)^2
    }.
\end{equation}
Since
\begin{equation}
    1-\alpha
    \le
    p(x\mid z=u)
    \le
    1+\alpha,
    \qquad u\in\{0,1\},
\end{equation}
we have
\begin{equation}
    \sup_{x\in[0,1]}\sup_{a\in[0,1]}
    \vert f_x'(a)\vert
    \le
    \frac{1+\alpha}{1-\alpha}
    =
    \kappa.
\end{equation}
Therefore,
\begin{equation}
    \sup_{x\in[0,1]}
    \vert \hat p_m(z=1\mid x)-p(z=1\mid x)\vert
    \le
    \kappa\vert \hat\mu_m-\mu\vert.
\end{equation}
Since
\begin{equation}
    \hat p_m(y=1\mid x)-p(y=1\mid x)
    =
    (1-2\rho_y)
    \left(
    \hat p_m(z=1\mid x)-p(z=1\mid x)
    \right),
\end{equation}
it follows that
\begin{equation}
    \sup_{x\in[0,1]}
    \vert \hat p_m(y=1\mid x)-p(y=1\mid x)\vert
    \le
    (1-2\rho_y)\kappa\vert \hat\mu_m-\mu\vert.
\end{equation}
Finally, since $\pi_{e_1},\dots,\pi_{e_m}\in[0,1]$ are iid, Hoeffding's inequality gives
\begin{equation}
    \Pr
    \left(
    \vert \hat\mu_m-\mu\vert>t
    \right)
    \le
    2\exp(-2mt^2).
\end{equation}
Substituting
\begin{equation}
    t=
    \frac{
        \epsilon
    }{
        (1-2\rho_y)\kappa
    }
\end{equation}
completes the proof.
\end{proof}

\subsection{Proof of Proposition \ref{prop:example_incorrect}} \label{proof:example_incorrect}

\begin{proof}
As in the previous proposition, for fixed $x$, denote
\begin{equation}
    f_x(a)
    \coloneqq
    \frac{
        a p(x\mid z=1)
    }{
        a p(x\mid z=1)+(1-a)p(x\mid z=0)
    }.
\end{equation}
Then,
\begin{equation}
    p(z=1\mid x)=f_x(\mu),
    \qquad
    \tilde p_m(z=1\mid x)=\frac{1}{m}\sum_{j=1}^m f_x(\pi_{e_j}).
\end{equation}
By the law of large numbers,
\begin{equation}
    \tilde p_m(z=1\mid x)
    \longrightarrow
    \E_{p(e)}[f_x(\pi_e)]
\end{equation}
almost surely. Moreover,
\begin{equation}
    f_x''(a)
    =
    -
    \frac{
        2p(x\mid z=1)p(x\mid z=0)
        \left(
        p(x\mid z=1)-p(x\mid z=0)
        \right)
    }{
        \left(
        a p(x\mid z=1)+(1-a)p(x\mid z=0)
        \right)^3
    }.
\end{equation}
Since
\begin{equation}
    p(x\mid z=1)-p(x\mid z=0)
    =
    2\alpha(2x-1),
\end{equation}
we have $f_x''(a)\neq 0$ for all $a\in[0,1]$ whenever $x\neq 1/2$. Therefore, $f_x$ is strictly convex or strictly concave. Since $\operatorname{Var}_{p(e)}(\pi_e)>0$, Jensen's inequality is strict:
\begin{equation}
    \E_{p(e)}[f_x(\pi_e)]
    \neq
    f_x(\E_{p(e)}[\pi_e])
    =
    f_x(\mu).
\end{equation}
Therefore, as $m \to \infty $, 
\begin{equation}
    \tilde p_m(z=1\mid x)
    \longrightarrow
    \E_{p(e)}[f_x(\pi_e)]
    \neq
    p(z=1\mid x).
\end{equation}
Since
\begin{equation}
    \tilde p_m(y=1\mid x)
    =
    \rho_y+(1-2\rho_y)\tilde p_m(z=1\mid x),
\end{equation}
and
\begin{equation}
    p(y=1\mid x)
    =
    \rho_y+(1-2\rho_y)p(z=1\mid x),
\end{equation}
with $1-2\rho_y>0$, it follows that
\begin{equation}
    \tilde p_m(y=1\mid x)
    \longrightarrow
    \rho_y+(1-2\rho_y)\E_{p(e)}[f_x(\pi_e)]
    \neq
    p(y=1\mid x).
\end{equation}
\end{proof}

\section{Additional Empirical Results} \label{sup:results}

\subsection{Sensitivity and Ablation Studies} \label{sup:ablation}

In this section we evaluate the robustness of EmBER through sensitivity and ablation studies in the parametric simulation. Starting from the baseline setting in Section~\ref{sec:param_sim}, we vary one data-generating parameter at a time: the strength of the observed branch cue $\rho$, the number of training environments $m$, the number of observations per environment $n$, and the observed covariate dimension $d$. We compare EmBER with ERM, IRM, V-REx, and Fishr, and include the optimal-Bayes oracle reference that uses the optimal predictor based on the observed covariates. 

We also evaluate EmBER variants that replace selected learned components with oracle or uniform alternatives, in order to isolate the role of the learned environment weights and the learned environment-specific posteriors. These additional variants are detailed in Table \ref{tab:eber_variants}.

\begin{table}[H]
    \centering
    \resizebox{1.00\textwidth}{!}{
    \small
    \begin{tabular}{c p{0.47\textwidth} c c c c}
    \toprule
    Variant
    & Name
    & $p(e\vert x)$ train
    & $p(e\vert x)$ test
    & Posterior train
    & Posterior test \\
    \midrule
    1
    & EmBER (learned-weights-learned-posteriors)
    & learned
    & learned
    & learned
    & learned \\
    2
    & EmBER (learned-train-weights-uniform-test-weights)
    & learned
    & uniform
    & learned
    & learned \\
    3
    & EmBER (learned-train-weights-oracle-test-weights)
    & learned
    & oracle
    & learned
    & learned \\
    4
    & EmBER (oracle-train-weights-oracle-test-weights)
    & oracle
    & oracle
    & learned
    & learned \\
    5
    & EmBER (learned-weights-oracle-posteriors)
    & learned
    & learned
    & oracle
    & oracle \\
    6
    & EmBER (uniform-train-weights-uniform-test-weights)
    & uniform
    & uniform
    & learned
    & learned \\
    7
    & EmBER (uniform-weights-oracle-posteriors)
    & uniform
    & uniform
    & oracle
    & oracle \\
    \bottomrule
    \end{tabular}
    }
    \vspace{0.5em}
    \caption{EmBER variants used in the ablation and sensitivity studies.}
    \label{tab:eber_variants}
\end{table}

\paragraph{Sensitivity analysis results} Tables \ref{tab:sensitivity_m1}-\ref{tab:sensitivity_6} show that across the sensitivity studies, EmBER remains the best method among the implementable methods in NLL, and usually also in accuracy. 

Increasing the \emph{branch-cue strength} leads to improvement in all methods, especially V-REx. This is true also for EmBER. The NLL of EmBER improves from the baseline value, and its accuracy reaches the optimal-Bayes accuracy.

Increasing the \emph{number of training environments} brings EmBER  closest to the optimal-Bayes: for $m=12$ and $m=36$, EmBER has NLL $0.254$ and $0.252$, compared with the optimal-Bayes $0.232$ and $0.234$. At $m=36$ the baselines also improve substantially, so the relative advantage of EmBER becomes much smaller.

Increasing the \emph{number of examples per environment} also helps EmBER: at $n=1000$, EmBER nearly matches the optimal-Bayes in accuracy and NLL, while maintaining a large gap over the competing methods.

Varying the \emph{data dimension} shows that at $d=50$ and $d=100$, EmBER is still clearly best among the implementable methods, but the gap to the optimal-Bayes widens. At $d=1000$, EmBER still has lower NLL than ERM, IRM, V-REx, and Fishr, but the accuracy of all methods collapses.

Runtime is affected most strongly by the number of training environments: EmBER increases from about $36$ seconds at $m=12$ to $173$ seconds at $m=36$, whereas changes in data dimension have much smaller effects on the relative runtime.

\paragraph{Ablation analysis results}
Variant 3, which uses learned training weights but oracle test weights, is almost identical to the optimal-Bayes predictor in accuracy and often improves slightly over EmBER in NLL. The slight improvement indicates that even when the approximation of $p(e\mid x)$ learned by EmBER is not perfect, the approximation of the per-environment posteriors is relatively accurate. Variant 4, which uses oracle weights at both train and test time, is usually even closer to optimal-Bayes, indicating that the remaining gap between EmBER and Bayes-x is largely due to imperfect estimation of the environment weights.

By contrast, variants 2 and 6 consistently yield inferior performance: replacing the test weights by uniform weights, or using uniform weights during both training and testing, brings performance close to the ERM/IRM/V-REx/Fishr range. The same holds for variant 7, showing that even when oracle posteriors are used, as long as uniform weights are used, the performance is poor. Variant 5 shows similar trends in the low- and moderate-dimensional settings, despite using oracle posteriors, which suggests that the way environment-specific components are aggregated is central.

Together, the ablations support the main modeling claim: EmBER’s gains come from learning observation-specific environment weights and using them consistently for aggregation.

\subsubsection{Sensitivity to Branch-Cue Strength}

\begin{table}[H]
    \centering
    \small
    \begin{tabular}{lccccc}
    \toprule
    Method & NLL $\downarrow$ & Accuracy $\uparrow$ & ECE $\downarrow$ & Time (s) & PV (NLL) \\
    \midrule
    Bayes-x & 0.231 $\pm$ 0.004 & 0.935 $\pm$ 0.001 & 0.002 $\pm$ 0.000 & 3.293 $\pm$ 0.414 & -- \\
    \midrule
    \midrule
    EmBER & \textbf{0.301 $\pm$ 0.071} & \textbf{0.935 $\pm$ 0.002} & 0.099 $\pm$ 0.082 & 17.563 $\pm$ 0.694 & -- \\
    ERM & 0.620 $\pm$ 0.012 & 0.572 $\pm$ 0.016 & 0.119 $\pm$ 0.028 & 4.304 $\pm$ 0.773 & 0.024 \\
    IRM & 0.606 $\pm$ 0.028 & 0.593 $\pm$ 0.032 & 0.115 $\pm$ 0.026 & 6.068 $\pm$ 0.966 & 0.024 \\
    V-REx & 0.580 $\pm$ 0.055 & 0.644 $\pm$ 0.097 & 0.122 $\pm$ 0.037 & 5.173 $\pm$ 0.676 & 0.024 \\
    Fishr & 0.603 $\pm$ 0.012 & 0.602 $\pm$ 0.031 & \textbf{0.089 $\pm$ 0.033} & 6.030 $\pm$ 0.626 & 0.024 \\
    \midrule
    \midrule
    EmBER--variant 2 & 0.656 $\pm$ 0.015 & 0.614 $\pm$ 0.036 & 0.152 $\pm$ 0.092 & 17.526 $\pm$ 0.670 & -- \\
    EmBER--variant 3 & 0.292 $\pm$ 0.071 & 0.935 $\pm$ 0.001 & 0.087 $\pm$ 0.087 & 0.000 $\pm$ 0.000 & -- \\
    EmBER--variant 4 & 0.242 $\pm$ 0.004 & 0.935 $\pm$ 0.001 & 0.010 $\pm$ 0.004 & 136.658 $\pm$ 8.059 & -- \\
    EmBER--variant 5 & 0.753 $\pm$ 0.083 & 0.633 $\pm$ 0.011 & 0.325 $\pm$ 0.067 & 207.051 $\pm$ 1.671 & -- \\
    EmBER--variant 6 & 0.664 $\pm$ 0.009 & 0.571 $\pm$ 0.022 & 0.064 $\pm$ 0.039 & 19.698 $\pm$ 1.036 & -- \\
    EmBER--variant 7 & 0.674 $\pm$ 0.009 & 0.589 $\pm$ 0.048 & 0.092 $\pm$ 0.031 & 178.919 $\pm$ 1.296 & -- \\
    \bottomrule
    \end{tabular}
    \vspace{0.5em}
    \caption{Sensitivity and ablation results for branch-cue strength:  $\rho = 0.15$.}
    \label{tab:sensitivity_m1}
\end{table}

\begin{table}[H]
    \centering
    \small
    \begin{tabular}{lccccc}
    \toprule
    Method & NLL $\downarrow$ & Accuracy $\uparrow$ & ECE $\downarrow$ & Time (s) & PV (NLL) \\
    \midrule
    Bayes-x & 0.231 $\pm$ 0.004 & 0.935 $\pm$ 0.001 & 0.002 $\pm$ 0.000 & 3.335 $\pm$ 0.267 & -- \\
    \midrule
    \midrule
    EmBER & \textbf{0.297 $\pm$ 0.067} & \textbf{0.935 $\pm$ 0.002} & 0.095 $\pm$ 0.080 & 17.390 $\pm$ 0.678 & -- \\
    ERM & 0.576 $\pm$ 0.026 & 0.626 $\pm$ 0.042 & 0.108 $\pm$ 0.056 & 4.602 $\pm$ 0.333 & 0.026 \\
    IRM & 0.551 $\pm$ 0.036 & 0.652 $\pm$ 0.048 & 0.123 $\pm$ 0.018 & 5.946 $\pm$ 0.956 & 0.026 \\
    V-REx & 0.544 $\pm$ 0.058 & 0.692 $\pm$ 0.081 & \textbf{0.084 $\pm$ 0.026} & 5.117 $\pm$ 0.405 & 0.026 \\
    Fishr & 0.561 $\pm$ 0.032 & 0.658 $\pm$ 0.026 & 0.093 $\pm$ 0.032 & 6.458 $\pm$ 0.686 & 0.026 \\
    \midrule
    \midrule
    EmBER--variant 2 & 0.657 $\pm$ 0.021 & 0.578 $\pm$ 0.066 & 0.142 $\pm$ 0.086 & 17.677 $\pm$ 0.736 & -- \\
    EmBER--variant 3 & 0.292 $\pm$ 0.066 & 0.935 $\pm$ 0.001 & 0.088 $\pm$ 0.081 & 0.000 $\pm$ 0.000 & -- \\
    EmBER--variant 4 & 0.242 $\pm$ 0.003 & 0.935 $\pm$ 0.001 & 0.011 $\pm$ 0.004 & 133.081 $\pm$ 9.282 & -- \\
    EmBER--variant 5 & 0.753 $\pm$ 0.082 & 0.628 $\pm$ 0.021 & 0.326 $\pm$ 0.065 & 196.114 $\pm$ 4.056 & -- \\
    EmBER--variant 6 & 0.650 $\pm$ 0.014 & 0.573 $\pm$ 0.011 & 0.106 $\pm$ 0.058 & 19.060 $\pm$ 0.292 & -- \\
    EmBER--variant 7 & 0.666 $\pm$ 0.010 & 0.553 $\pm$ 0.066 & 0.146 $\pm$ 0.034 & 170.055 $\pm$ 0.875 & -- \\
    \bottomrule
    \end{tabular}
    \vspace{0.5em}
    \caption{Sensitivity and ablation results for branch-cue strength:  $\rho = 0.20$.}
    \label{tab:sensitivity_0}
\end{table}

\subsubsection{Sensitivity to Number of Training Environments}

\begin{table}[H]
    \centering
    \small
    \begin{tabular}{lccccc}
    \toprule
    Method & NLL $\downarrow$ & Accuracy $\uparrow$ & ECE $\downarrow$ & Time (s) & PV (NLL) \\
    \midrule
    Bayes-x & 0.232 $\pm$ 0.005 & 0.935 $\pm$ 0.002 & 0.003 $\pm$ 0.002 & 3.306 $\pm$ 0.703 & -- \\
    \midrule
    \midrule
    EmBER & \textbf{0.254 $\pm$ 0.003} & \textbf{0.932 $\pm$ 0.002} & \textbf{0.026 $\pm$ 0.005} & 36.465 $\pm$ 1.990 & -- \\
    ERM & 0.572 $\pm$ 0.015 & 0.629 $\pm$ 0.003 & 0.099 $\pm$ 0.031 & 6.148 $\pm$ 1.002 & 0.003 \\
    IRM & 0.544 $\pm$ 0.026 & 0.663 $\pm$ 0.034 & 0.113 $\pm$ 0.032 & 11.869 $\pm$ 1.147 & 0.006 \\
    V-REx & 0.549 $\pm$ 0.051 & 0.675 $\pm$ 0.047 & 0.119 $\pm$ 0.021 & 7.431 $\pm$ 0.218 & 0.011 \\
    Fishr & 0.527 $\pm$ 0.047 & 0.706 $\pm$ 0.070 & 0.105 $\pm$ 0.015 & 12.114 $\pm$ 0.616 & 0.011 \\
    \midrule
    \midrule
    EmBER--variant 2 & 0.639 $\pm$ 0.028 & 0.594 $\pm$ 0.074 & 0.104 $\pm$ 0.001 & 37.461 $\pm$ 2.338 & -- \\
    EmBER--variant 3 & 0.247 $\pm$ 0.004 & 0.934 $\pm$ 0.002 & 0.016 $\pm$ 0.005 & 0.000 $\pm$ 0.000 & -- \\
    EmBER--variant 4 & 0.242 $\pm$ 0.006 & 0.935 $\pm$ 0.002 & 0.009 $\pm$ 0.002 & 302.468 $\pm$ 5.353 & -- \\
    EmBER--variant 5 & 0.674 $\pm$ 0.023 & 0.593 $\pm$ 0.081 & 0.186 $\pm$ 0.071 & 638.964 $\pm$ 19.802 & -- \\
    EmBER--variant 6 & 0.674 $\pm$ 0.006 & 0.573 $\pm$ 0.015 & 0.039 $\pm$ 0.016 & 37.607 $\pm$ 1.260 & -- \\
    EmBER--variant 7 & 0.694 $\pm$ 0.011 & 0.573 $\pm$ 0.064 & 0.110 $\pm$ 0.014 & 519.677 $\pm$ 4.796 & -- \\
    \bottomrule
    \end{tabular}
    \vspace{0.5em}
    \caption{Sensitivity and ablation results for number of training environments: $m=12$.}
    \label{tab:sensitivity_1}
\end{table}

\begin{table}[H]
    \centering
    \small
    \begin{tabular}{lccccc}
    \toprule
    Method & NLL $\downarrow$ & Accuracy $\uparrow$ & ECE $\downarrow$ & Time (s) & PV (NLL) \\
    \midrule
    Bayes-x & 0.234 $\pm$ 0.005 & 0.934 $\pm$ 0.002 & 0.003 $\pm$ 0.000 & 2.867 $\pm$ 0.036 & -- \\
    \midrule
    \midrule
    EmBER & \textbf{0.252 $\pm$ 0.005} & \textbf{0.932 $\pm$ 0.002} & \textbf{0.021 $\pm$ 0.003} & 173.110 $\pm$ 2.214 & -- \\
    ERM & 0.265 $\pm$ 0.009 & 0.931 $\pm$ 0.004 & 0.035 $\pm$ 0.006 & 10.531 $\pm$ 0.139 & 0.056 \\
    IRM & 0.312 $\pm$ 0.007 & 0.926 $\pm$ 0.003 & 0.105 $\pm$ 0.008 & 53.156 $\pm$ 0.952 & 0.024 \\
    V-REx & 0.296 $\pm$ 0.008 & 0.931 $\pm$ 0.002 & 0.094 $\pm$ 0.006 & 21.215 $\pm$ 0.475 & 0.024 \\
    Fishr & 0.273 $\pm$ 0.002 & 0.930 $\pm$ 0.002 & 0.047 $\pm$ 0.007 & 39.614 $\pm$ 0.927 & 0.035 \\
    \midrule
    \midrule
    EmBER--variant 2 & 0.641 $\pm$ 0.016 & 0.643 $\pm$ 0.013 & 0.097 $\pm$ 0.006 & 172.446 $\pm$ 3.578 & -- \\
    EmBER--variant 3 & 0.249 $\pm$ 0.006 & 0.933 $\pm$ 0.002 & 0.017 $\pm$ 0.005 & 0.000 $\pm$ 0.000 & -- \\
    EmBER--variant 4 & 0.241 $\pm$ 0.005 & 0.934 $\pm$ 0.002 & 0.011 $\pm$ 0.002 & 1486.899 $\pm$ 45.133 & -- \\
    EmBER--variant 5 & 0.661 $\pm$ 0.001 & 0.639 $\pm$ 0.003 & 0.187 $\pm$ 0.057 & 4185.989 $\pm$ 205.771 & -- \\
    EmBER--variant 6 & 0.673 $\pm$ 0.006 & 0.573 $\pm$ 0.041 & 0.053 $\pm$ 0.027 & 160.284 $\pm$ 0.768 & -- \\
    EmBER--variant 7 & 0.688 $\pm$ 0.000 & 0.615 $\pm$ 0.011 & 0.109 $\pm$ 0.011 & 3227.361 $\pm$ 30.513 & -- \\
    \bottomrule
    \end{tabular}
    \vspace{0.5em}
    \caption{Sensitivity and ablation results for number of training environments: $m=36$.}
    \label{tab:sensitivity_2}
\end{table}

\subsubsection{Sensitivity to Number of Training Examples per Environment}

\begin{table}[H]
    \centering
    \small
    \begin{tabular}{lccccc}
    \toprule
    Method & NLL $\downarrow$ & Accuracy $\uparrow$ & ECE $\downarrow$ & Time (s) & PV (NLL) \\
    \midrule
    Bayes-x & 0.233 $\pm$ 0.004 & 0.934 $\pm$ 0.002 & 0.002 $\pm$ 0.001 & 6.090 $\pm$ 0.769 & -- \\
    \midrule
    \midrule
    EmBER & \textbf{0.250 $\pm$ 0.006} & \textbf{0.932 $\pm$ 0.002} & \textbf{0.019 $\pm$ 0.005} & 35.630 $\pm$ 0.307 & -- \\
    ERM & 0.546 $\pm$ 0.003 & 0.665 $\pm$ 0.038 & 0.100 $\pm$ 0.038 & 8.766 $\pm$ 0.835 & $< 0.001$ \\
    IRM & 0.544 $\pm$ 0.038 & 0.658 $\pm$ 0.027 & 0.110 $\pm$ 0.028 & 12.393 $\pm$ 0.189 & 0.007 \\
    V-REx & 0.509 $\pm$ 0.081 & 0.740 $\pm$ 0.098 & 0.095 $\pm$ 0.007 & 9.420 $\pm$ 0.916 & 0.032 \\
    Fishr & 0.510 $\pm$ 0.031 & 0.713 $\pm$ 0.054 & 0.126 $\pm$ 0.036 & 12.524 $\pm$ 0.180 & 0.007 \\
    \midrule
    \midrule
    EmBER--variant 2 & 0.654 $\pm$ 0.005 & 0.546 $\pm$ 0.055 & 0.155 $\pm$ 0.003 & 34.958 $\pm$ 0.631 & -- \\
    EmBER--variant 3 & 0.246 $\pm$ 0.005 & 0.933 $\pm$ 0.002 & 0.015 $\pm$ 0.004 & 0.000 $\pm$ 0.000 & -- \\
    EmBER--variant 4 & 0.243 $\pm$ 0.005 & 0.934 $\pm$ 0.002 & 0.006 $\pm$ 0.004 & 285.632 $\pm$ 2.888 & -- \\
    EmBER--variant 5 & 0.662 $\pm$ 0.000 & 0.639 $\pm$ 0.001 & 0.268 $\pm$ 0.030 & 465.074 $\pm$ 9.839 & -- \\
    EmBER--variant 6 & 0.681 $\pm$ 0.002 & 0.571 $\pm$ 0.005 & 0.045 $\pm$ 0.011 & 39.347 $\pm$ 0.542 & -- \\
    EmBER--variant 7 & 0.688 $\pm$ 0.000 & 0.589 $\pm$ 0.057 & 0.081 $\pm$ 0.056 & 402.842 $\pm$ 8.071 & -- \\
    \bottomrule
    \end{tabular}
    \vspace{0.5em}
    \caption{Sensitivity and ablation results for training examples per environment: $n=1000$.}
    \label{tab:sensitivity_3}
\end{table}

\subsubsection{Sensitivity to Data Dimension}

\begin{table}[H]
    \centering
    \small
    \begin{tabular}{lccccc}
    \toprule
    Method & NLL $\downarrow$ & Accuracy $\uparrow$ & ECE $\downarrow$ & Time (s) & PV (NLL) \\
    \midrule
    Bayes-x & 0.234 $\pm$ 0.002 & 0.934 $\pm$ 0.001 & 0.002 $\pm$ 0.001 & 3.438 $\pm$ 0.587 & -- \\
    \midrule
    \midrule
    EmBER & \textbf{0.366 $\pm$ 0.101} & \textbf{0.907 $\pm$ 0.007} & 0.126 $\pm$ 0.099 & 17.741 $\pm$ 0.333 & -- \\
    ERM & 0.685 $\pm$ 0.003 & 0.570 $\pm$ 0.002 & 0.071 $\pm$ 0.001 & 3.988 $\pm$ 0.340 & 0.035 \\
    IRM & 0.688 $\pm$ 0.001 & 0.564 $\pm$ 0.003 & 0.057 $\pm$ 0.005 & 5.992 $\pm$ 0.755 & 0.035 \\
    V-REx & 0.688 $\pm$ 0.004 & 0.561 $\pm$ 0.005 & \textbf{0.052 $\pm$ 0.006} & 5.470 $\pm$ 0.996 & 0.035 \\
    Fishr & 0.689 $\pm$ 0.000 & 0.570 $\pm$ 0.001 & 0.077 $\pm$ 0.002 & 5.720 $\pm$ 0.314 & 0.035 \\
    \midrule
    \midrule
    EmBER--variant 2 & 0.666 $\pm$ 0.012 & 0.560 $\pm$ 0.020 & 0.105 $\pm$ 0.085 & 17.809 $\pm$ 0.418 & -- \\
    EmBER--variant 3 & 0.322 $\pm$ 0.113 & 0.932 $\pm$ 0.001 & 0.106 $\pm$ 0.130 & 0.000 $\pm$ 0.000 & -- \\
    EmBER--variant 4 & 0.258 $\pm$ 0.008 & 0.934 $\pm$ 0.001 & 0.028 $\pm$ 0.008 & 141.465 $\pm$ 0.622 & -- \\
    EmBER--variant 5 & 0.758 $\pm$ 0.089 & 0.633 $\pm$ 0.009 & 0.313 $\pm$ 0.071 & 224.025 $\pm$ 3.412 & -- \\
    EmBER--variant 6 & 0.673 $\pm$ 0.003 & 0.571 $\pm$ 0.003 & 0.022 $\pm$ 0.009 & 20.482 $\pm$ 0.565 & -- \\
    EmBER--variant 7 & 0.687 $\pm$ 0.002 & 0.525 $\pm$ 0.044 & 0.045 $\pm$ 0.014 & 199.823 $\pm$ 1.191 & -- \\
    \bottomrule
    \end{tabular}
    \vspace{0.5em}
    \caption{Sensitivity and ablation results for data dimension: $d=50$.}
    \label{tab:sensitivity_4}
\end{table}

\begin{table}[H]
    \centering
    \small
    \begin{tabular}{lccccc}
    \toprule
    Method & NLL $\downarrow$ & Accuracy $\uparrow$ & ECE $\downarrow$ & Time (s) & PV (NLL) \\
    \midrule
    Bayes-x & 0.232 $\pm$ 0.001 & 0.935 $\pm$ 0.000 & 0.003 $\pm$ 0.001 & 3.345 $\pm$ 0.401 & -- \\
    \midrule
    \midrule
    EmBER & \textbf{0.386 $\pm$ 0.078} & \textbf{0.882 $\pm$ 0.015} & 0.108 $\pm$ 0.069 & 19.376 $\pm$ 1.562 & -- \\
    ERM & 0.734 $\pm$ 0.003 & 0.562 $\pm$ 0.006 & 0.122 $\pm$ 0.005 & 3.921 $\pm$ 0.036 & 0.021 \\
    IRM & 0.735 $\pm$ 0.014 & 0.554 $\pm$ 0.005 & 0.118 $\pm$ 0.016 & 6.359 $\pm$ 0.714 & 0.021 \\
    V-REx & 0.726 $\pm$ 0.008 & 0.553 $\pm$ 0.005 & \textbf{0.108 $\pm$ 0.007} & 5.380 $\pm$ 0.700 & 0.021 \\
    Fishr & 0.744 $\pm$ 0.002 & 0.558 $\pm$ 0.005 & 0.132 $\pm$ 0.001 & 6.033 $\pm$ 0.757 & 0.021 \\
    \midrule
    \midrule
    EmBER--variant 2 & 0.671 $\pm$ 0.008 & 0.579 $\pm$ 0.008 & 0.057 $\pm$ 0.015 & 18.683 $\pm$ 1.345 & -- \\
    EmBER--variant 3 & 0.310 $\pm$ 0.091 & 0.933 $\pm$ 0.002 & 0.103 $\pm$ 0.103 & 0.000 $\pm$ 0.000 & -- \\
    EmBER--variant 4 & 0.277 $\pm$ 0.012 & 0.929 $\pm$ 0.005 & 0.045 $\pm$ 0.007 & 140.905 $\pm$ 1.172 & -- \\
    EmBER--variant 5 & 0.760 $\pm$ 0.089 & 0.636 $\pm$ 0.001 & 0.305 $\pm$ 0.080 & 226.697 $\pm$ 2.373 & -- \\
    EmBER--variant 6 & 0.686 $\pm$ 0.002 & 0.561 $\pm$ 0.003 & 0.039 $\pm$ 0.008 & 20.606 $\pm$ 0.696 & -- \\
    EmBER--variant 7 & 0.688 $\pm$ 0.002 & 0.529 $\pm$ 0.052 & 0.050 $\pm$ 0.022 & 197.960 $\pm$ 3.629 & -- \\
    \bottomrule
    \end{tabular}
    \vspace{0.5em}
    \caption{Sensitivity and ablation results for data dimension: $d=100$.}
    \label{tab:sensitivity_5}
\end{table}

\begin{table}[H]
    \small
    \begin{tabular}{lccccc}
    \toprule
    Method & NLL $\downarrow$ & Accuracy $\uparrow$ & ECE $\downarrow$ & Time (s) & PV (NLL) \\
    \midrule
    Bayes-x & 0.232 $\pm$ 0.002 & 0.934 $\pm$ 0.001 & 0.002 $\pm$ 0.001 & 3.055 $\pm$ 0.051 & -- \\
    \midrule
    \midrule
    EmBER & \textbf{1.773 $\pm$ 0.331} & 0.530 $\pm$ 0.008 & \textbf{0.349 $\pm$ 0.026} & 26.296 $\pm$ 0.713 & -- \\
    ERM & 2.169 $\pm$ 0.076 & 0.531 $\pm$ 0.007 & 0.378 $\pm$ 0.002 & 5.045 $\pm$ 0.509 & 0.213 \\
    IRM & 2.117 $\pm$ 0.018 & 0.532 $\pm$ 0.005 & 0.377 $\pm$ 0.006 & 7.186 $\pm$ 0.666 & 0.213 \\
    V-REx & 2.175 $\pm$ 0.051 & 0.532 $\pm$ 0.007 & 0.379 $\pm$ 0.008 & 5.240 $\pm$ 0.129 & 0.213 \\
    Fishr & 2.144 $\pm$ 0.098 & \textbf{0.532 $\pm$ 0.007} & 0.378 $\pm$ 0.007 & 7.724 $\pm$ 0.294 & 0.213 \\
    \midrule
    \midrule
    EmBER--variant 2 & 1.209 $\pm$ 0.081 & 0.529 $\pm$ 0.003 & 0.279 $\pm$ 0.027 & 25.465 $\pm$ 0.879 & -- \\
    EmBER--variant 3 & 1.167 $\pm$ 0.144 & 0.631 $\pm$ 0.019 & 0.240 $\pm$ 0.018 & 0.000 $\pm$ 0.000 & -- \\
    EmBER--variant 4 & 1.702 $\pm$ 0.228 & 0.616 $\pm$ 0.019 & 0.298 $\pm$ 0.024 & 151.139 $\pm$ 0.521 & -- \\
    EmBER--variant 5 & 0.748 $\pm$ 0.064 & 0.572 $\pm$ 0.005 & 0.130 $\pm$ 0.079 & 223.441 $\pm$ 2.586 & -- \\
    EmBER--variant 6 & 1.970 $\pm$ 0.284 & 0.530 $\pm$ 0.006 & 0.353 $\pm$ 0.025 & 27.924 $\pm$ 0.794 & -- \\
    EmBER--variant 7 & 0.688 $\pm$ 0.003 & 0.533 $\pm$ 0.059 & 0.056 $\pm$ 0.027 & 195.873 $\pm$ 1.272 & -- \\
    \bottomrule
    \end{tabular}
    \centering
    \vspace{0.5em}
    \caption{Sensitivity and ablation results for data dimension: $d=1000$.}
    \label{tab:sensitivity_6}
\end{table}

\subsection{Additional Results for Main-Text Experiments} \label{sup:details}

Tables \ref{tab:add_parametric}--\ref{tab:add_sepsis} report the main-text results, augmented with ECE and running time. The times are recorded on a single CPU 5 performance cores and 6 efficiency cores, for all experiments except Colored MNIST where a single 
A100 GPU was used.

\begin{table}[H]
    \small
    \centering
    \begin{tabular}{lccccc}
    \toprule
    Method & NLL $\downarrow$ & Accuracy $\uparrow$ & ECE $\downarrow$ & Time (s) & PV (NLL) \\
    \midrule
    Bayes-x & 0.233 $\pm$ 0.004 & 0.934 $\pm$ 0.002 & 0.002 $\pm$ 0.001 & 3.090 $\pm$ 0.379 & -- \\
    \midrule
    EmBER & \textbf{0.319 $\pm$ 0.069} & \textbf{0.928 $\pm$ 0.003} & 0.112 $\pm$ 0.078 & 18.920 $\pm$ 1.117 & -- \\
    ERM & 0.652 $\pm$ 0.005 & 0.577 $\pm$ 0.009 & 0.103 $\pm$ 0.013 & 4.373 $\pm$ 0.954 & 0.015 \\
    IRM & 0.654 $\pm$ 0.006 & 0.573 $\pm$ 0.012 & \textbf{0.086 $\pm$ 0.022} & 6.053 $\pm$ 0.789 & 0.015 \\
    V-REx & 0.640 $\pm$ 0.026 & 0.572 $\pm$ 0.016 & 0.109 $\pm$ 0.021 & 5.008 $\pm$ 0.658 & 0.015 \\
    Fishr & 0.652 $\pm$ 0.006 & 0.582 $\pm$ 0.011 & 0.097 $\pm$ 0.010 & 5.812 $\pm$ 0.712 & 0.015 \\
    \bottomrule
    \end{tabular}
    \caption{Additional results for the parametric simulation with baseline values.}
    \label{tab:add_parametric}
\end{table}

\begin{table}[H]
    \centering
    \small
    \begin{tabular}{lccccc}
    \toprule
    Method & NLL $\downarrow$ & Accuracy $\uparrow$ & ECE $\downarrow$ & Time (s) & PV (NLL) \\
    \midrule
    EmBER & \textbf{0.519 $\pm$ 0.077} & \textbf{0.787 $\pm$ 0.071} & \textbf{0.076 $\pm$ 0.021} & 149.269 $\pm$ 37.363 & -- \\
    ERM & 0.863 $\pm$ 0.090 & 0.650 $\pm$ 0.021 & 0.194 $\pm$ 0.027 & 45.019 $\pm$ 1.007 & 0.004 \\
    IRM & 0.572 $\pm$ 0.007 & 0.747 $\pm$ 0.017 & 0.101 $\pm$ 0.014 & 0.041 $\pm$ 0.006 & 0.203 \\
    V-REx & 0.605 $\pm$ 0.019 & 0.684 $\pm$ 0.011 & 0.081 $\pm$ 0.011 & 70.164 $\pm$ 0.623 & 0.140 \\
    Fishr & 0.917 $\pm$ 0.061 & 0.639 $\pm$ 0.020 & 0.207 $\pm$ 0.021 & 101.795 $\pm$ 1.436 & 0.004 \\
    \bottomrule
    \end{tabular}
    \vspace{0.5em}
    \caption{Additional results for the colored-MNIST simulation.}
    \label{tab:add_mnist}
\end{table}

\begin{table}[H]
    \centering
    \small
    \begin{tabular}{lccccc}
    \toprule
    Method & NLL $\downarrow$ & Accuracy $\uparrow$ & ECE $\downarrow$ & Time (s) & PV (NLL) \\
    \midrule
    EmBER & \textbf{0.221 $\pm$ 0.019} & \textbf{0.921 $\pm$ 0.007} & \textbf{0.048 $\pm$ 0.011} & 32.282 $\pm$ 0.690 & -- \\
    ERM & 0.390 $\pm$ 0.001 & 0.897 $\pm$ 0.001 & 0.104 $\pm$ 0.002 & 4.293 $\pm$ 1.548 & $< 0.001$ \\
    Fishr & 0.384 $\pm$ 0.004 & 0.894 $\pm$ 0.001 & 0.111 $\pm$ 0.002 & 6.666 $\pm$ 0.227 & $< 0.001$ \\
    IRM & 0.383 $\pm$ 0.009 & 0.891 $\pm$ 0.003 & 0.115 $\pm$ 0.006 & 6.781 $\pm$ 0.490 & $< 0.001$ \\
    VREx & 0.382 $\pm$ 0.009 & 0.892 $\pm$ 0.003 & 0.114 $\pm$ 0.005 & 5.154 $\pm$ 0.271 & $< 0.001$ \\
    \bottomrule
    \end{tabular}
    \vspace{0.5em}
    \caption{Additional results for the quasar-star classification experiment.}
    \label{tab:add_sdss}
\end{table}

\begin{table}[H]
    \centering
    \small
    \begin{tabular}{lccccc}
    \toprule
    Method & NLL $\downarrow$ & Accuracy $\uparrow$ & ECE $\downarrow$ & Time (s) & PV (NLL) \\
    \midrule
    EmBER & \textbf{0.850 $\pm$ 0.112} & \textbf{0.591 $\pm$ 0.027} & \textbf{0.131 $\pm$ 0.050} & 1.163 $\pm$ 0.076 & -- \\
    ERM & 1.685 $\pm$ 0.136 & 0.587 $\pm$ 0.016 & 0.287 $\pm$ 0.018 & 0.242 $\pm$ 0.274 & 0.002 \\
    Fishr & 1.449 $\pm$ 0.226 & 0.582 $\pm$ 0.021 & 0.255 $\pm$ 0.035 & 0.191 $\pm$ 0.016 & 0.008 \\
    IRM & 1.548 $\pm$ 0.224 & 0.588 $\pm$ 0.011 & 0.265 $\pm$ 0.025 & 0.199 $\pm$ 0.015 & 0.007 \\
    VREx & 0.903 $\pm$ 0.087 & 0.571 $\pm$ 0.014 & 0.157 $\pm$ 0.026 & 0.153 $\pm$ 0.009 & 0.543 \\
    \bottomrule
    \end{tabular}
    \vspace{0.5em}
    \caption{Additional results for the microbiome classification experiment.}
    \label{tab:add_micro}
\end{table}

\begin{table}[H]
    \centering
    \small
    \begin{tabular}{lccccc}
    \toprule
    Method & AUROC $\uparrow$ & AUPRC $\uparrow$ & ECE $\downarrow$ & Time (s) & PV (AUROC) \\
    \midrule
    EmBER & \textbf{0.740 $\pm$ 0.008} & \textbf{0.179 $\pm$ 0.008} & 0.028 $\pm$ 0.009 & 43.754 $\pm$ 0.921 & -- \\
    ERM & 0.692 $\pm$ 0.017 & 0.153 $\pm$ 0.013 & 0.039 $\pm$ 0.004 & 5.718 $\pm$ 0.348 & 0.003 \\
    Fishr & 0.697 $\pm$ 0.014 & 0.156 $\pm$ 0.008 & 0.039 $\pm$ 0.005 & 12.690 $\pm$ 0.425 & 0.002 \\
    IRM & 0.724 $\pm$ 0.012 & 0.169 $\pm$ 0.012 & \textbf{0.023 $\pm$ 0.004} & 13.422 $\pm$ 0.671 & 0.112 \\
    VREx & 0.697 $\pm$ 0.012 & 0.154 $\pm$ 0.011 & 0.029 $\pm$ 0.007 & 8.814 $\pm$ 0.165 & 0.001 \\
    \bottomrule
    \end{tabular}
    \vspace{0.5em}
    \caption{Additional results for the sepsis prediction experiment.}
    \label{tab:add_sepsis}
\end{table}

\subsection{Method-Agreement Analysis} \label{sup:error}

To better understand where EmBER differs from the competing methods, we decompose test examples by method agreement. In Figures \ref{fig:SDSS_agreement} and \ref{fig:ZENODO_agreement} the left panels report the fraction of test examples on which all methods either succeed or fail. The right panels focus on disagreement cases where EmBER is correct, and report how often each competing method is also correct on those examples. In both datasets, no competing method succeeds consistently on the EmBER-success disagreement subset, indicating that these cases are not simply easy examples solved by all methods.

\begin{figure}[ht]
    \centering
    \includegraphics[width=0.75\linewidth]{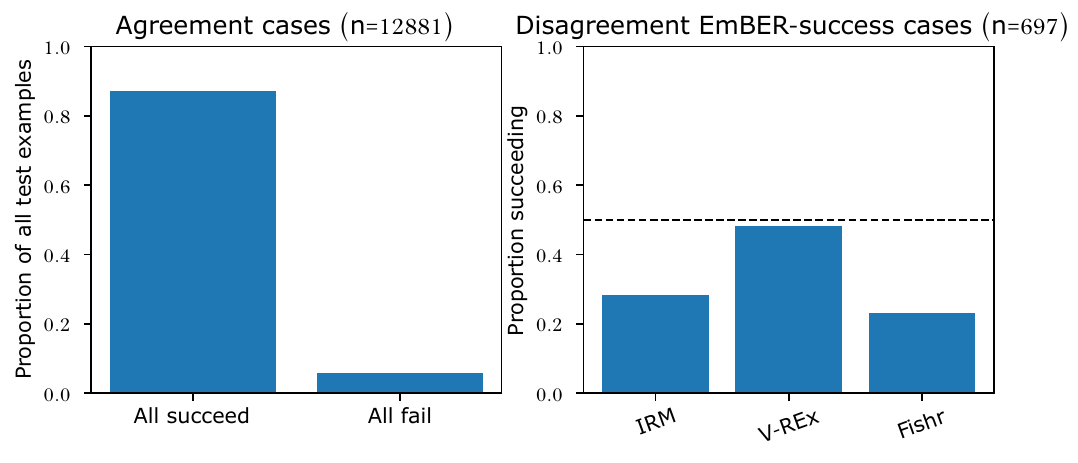}
    \caption{Agreement between methods in the quasar-star  classification experiment. Left: proportion of test examples on which all methods are correct or all methods are incorrect.
    Right: among disagreement cases where EmBER is correct, the proportion also classified correctly by each competing method. 
    The dashed line marks $0.5$.}
    \label{fig:SDSS_agreement}
\end{figure}

\begin{figure}[ht]
    \centering
    \includegraphics[width=0.75\linewidth]{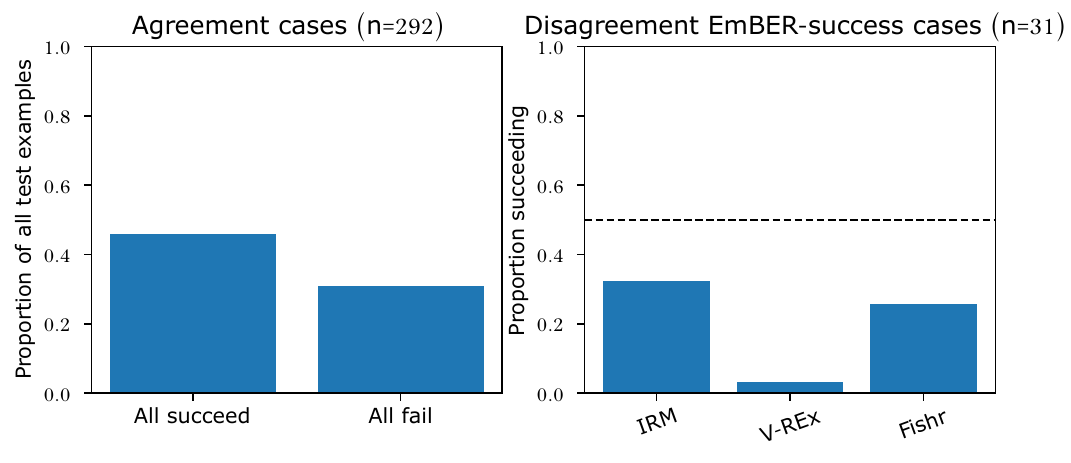}
    \caption{Agreement between methods in the microbiome classification experiment. Left: proportion of test examples on which all methods are correct or all methods are incorrect.
    Right: among disagreement cases where EmBER is correct, the proportion also classified correctly by each competing method. 
    The dashed line marks $0.5$.}
    \label{fig:ZENODO_agreement}
\end{figure}

\section{Additional Experimental Details} \label{sup:exp}

\subsection{Motivating Example} \label{sup:exp_example}

For Figure \ref{fig:example}, we set $\alpha=0.85$ and $\rho_y=0.10$. Environment-specific prevalences were sampled as $\pi_e\sim\operatorname{Beta}(0.5,0.5)$, so that $\mu=\E_{p(e)}[\pi_e]=0.5$. For the first two panels, we sampled $m=30$ observed environments and evaluated the conditional weights at $x\in\{0.15,0.50,0.85\}$. For the convergence panel, we used $m\in\{20,50,100,200,500,1000,2000,5000\}$ observed environments and repeated the experiment 500 times. The uniform-weight limiting error was approximated by Monte Carlo integration using 200,000 draws of $\pi_e$. All curves were evaluated on an equally spaced grid of 500 values in $[0.001,0.999]$.

\subsection{Parametric Model}

The simulation is designed so that environments affect the observed data only through the latent variable $z$. Each environment $j$ is associated with a branch index $b_j$. We use branch values
\begin{equation}
b_j\in\{-3,-2,-1,0,1,2\}.
\end{equation}
For each environment, we draw a branch center
\begin{equation}
m_j
=
\Delta b_j+\phi_0+\delta_j,
\qquad
\delta_j\sim\mathcal{N}(0,\sigma_\delta^2).
\end{equation}
Given $m_j$, observations in that environment are generated independently as
\begin{equation}
z_{ij}\mid e_j
\sim
\mathcal{N}(m_j,\sigma_z^2).
\end{equation}
The observed covariates are generated from $z_{ij}$ by
\begin{equation}
x_{ij,1}
=
\cos(z_{ij})+\varepsilon_{ij,1},
\end{equation}
\begin{equation}
x_{ij,2}
=
\sin(z_{ij})+\varepsilon_{ij,2},
\end{equation}
\begin{equation}
x_{ij,3}
=
\rho z_{ij}+\varepsilon_{ij,3},
\end{equation}
and
\begin{equation}
x_{ij,k}
=
\varepsilon_{ij,k},
\qquad
k=4,\ldots,d.
\end{equation}
The noise variables satisfy
\begin{equation}
\varepsilon_{ij,1},\varepsilon_{ij,2}
\sim
\mathcal{N}(0,\sigma_{\mathrm{phase}}^2),
\end{equation}
\begin{equation}
\varepsilon_{ij,3}
\sim
\mathcal{N}(0,\sigma_{\mathrm{branch}}^2),
\end{equation}
and
\begin{equation}
\varepsilon_{ij,k}
\sim
\mathcal{N}(0,\sigma_{\mathrm{noise}}^2),
\qquad
k=4,\ldots,d.
\end{equation}
The label is generated by
\begin{equation}
y_{ij}\mid z_{ij}
\sim
\mathrm{Bernoulli}
\left(
\sigma\left(\beta_0+\beta_1\sin(z_{ij}/2)\right)
\right),
\end{equation}
where $\sigma(t)=(1+\exp(-t))^{-1}$. Therefore,
\begin{equation}
p(e,z,x,y)
=
p(e)\,p(z\mid e)\,p(x\mid z)\,p(y\mid z)
\end{equation}
as in the assumed model.

The key feature of this construction is that $(x_1,x_2)$ reveal the phase of $z$ but not its branch. For every integer $k$,
\begin{equation}
\cos(z+2\pi k)=\cos(z),
\qquad
\sin(z+2\pi k)=\sin(z).
\end{equation}
However, adjacent branches have opposite label mechanisms because
\begin{equation}
\sin((z+2\pi)/2)
=
-\sin(z/2).
\end{equation}
Thus, $p(y\mid x,e)$ varies across environments, even though the latent mechanism $p(y\mid z)$ is stable.

The baseline parameters are
\begin{equation}
m_{\mathrm{train}}=6,
\qquad
n=500,
\qquad
d=10,
\end{equation}
\begin{equation}
\Delta=2\pi,
\qquad
\phi_0=\pi/2,
\qquad
\sigma_\delta=0.05,
\end{equation}
and
\begin{equation}
\sigma_z=0.35,
\qquad
\sigma_{\mathrm{phase}}=0.05,
\qquad
\rho=0.1,
\qquad
\sigma_{\mathrm{branch}}=0.1,
\qquad
\sigma_{\mathrm{noise}}=1.
\end{equation}
The label parameter is
\begin{equation}
\beta=4.
\end{equation}
The validation and test environments use the same number of observations per environment as the training environments.

\paragraph{Sensitivity analysis}
We study the effect of each simulation parameter separately. Starting from the baseline configuration, we vary one parameter at a time:
\begin{equation}
\rho\in\{0.1,0.15,0.2\},
\end{equation}
\begin{equation}
m_{\mathrm{train}}\in\{6,12,36\},
\end{equation}
\begin{equation}
n\in\{100,500,1000\},
\end{equation}
\begin{equation}
d\in\{10,50,100,1000\}.
\end{equation}

\paragraph{Training, validation, and test environments}
For every configuration, we generate three independent collections of environments: training environments, validation environments, and test environments. The training set contains $m_{\mathrm{train}}$ environments. The validation set contains $12$ new environments, and the test set contains $50$ new environments. All validation and test environments are sampled from the same environment population as the training environments.

The validation environments are used only for selecting regularization coefficients for IRM, VREx, and Fishr over a grid search of $\lambda \in \{0.1, 1, 10\}$.

\paragraph{Models}
All methods receive only the observed covariates $x$, labels $y$, and training-environment labels $e$ when required by the method. The latent variable $z$ is used only to compute oracle reference performance. All methods use the same MLP  with one hidden layer of width $16$ and ReLU activations, followed by a $32$-dimensional linear layer. The latent dimension of EmBER was set to 1.

For EmBER, we include the environment-classification term for estimating $\hat p(e\mid x)$. 
We evaluate two EmBER variants. The first is the full EmBER predictor, which uses learned environment weights $\hat p(e\mid x)$. The second is a uniform-weight ablation, which replaces the learned environment weights by uniform weights at prediction time. 

\paragraph{Covariate optimal-Bayes oracle reference}

We report the Bayes-optimal predictor based on the observed covariates $x$:
\begin{equation}
p(y=1\mid x)
=
\int
p(y=1\mid z)p(z\mid x) dz.
\end{equation}
Here, $p(z\mid x)$ is not available in closed form since $x_1$ and $x_2$ depend on $\cos(z)$ and $\sin(z)$. We therefore compute the Bayes predictor numerically on a dense grid over $z$, using the true data-generating parameters.

\paragraph{Optimization and evaluation}
All methods are trained for $10$ epochs using Adam with learning rate $10^{-3}$ and batch size $64$. EmBER predictions are estimated with $1000$ Monte Carlo samples. Expected calibration error is computed using 10 equally spaced probability bins.

\subsection{Colored MNIST} \label{sup:mnist_details}

For data generation we use $\alpha=\beta=0.95$ and set
\begin{align*}
& p(y=0\vert z=1)=p(y=0\vert z=2)=\alpha,
\qquad
p(y=1\vert z=3)=p(y=1\vert z=4)=\alpha,\\
& p(c=\mathrm{red}\vert z=1)=p(c=\mathrm{red}\vert z=3)=\beta,
\qquad
p(c=\mathrm{green}\vert z=2)=p(c=\mathrm{green}\vert z=4)=\beta.
\end{align*}

We set prototype-specific digit probabilities to
\begin{align*}
p(d\vert z=1) &= (0.30,0.22,0.22,0.10,0.08,0.06,0.02,0,0,0),\\
p(d\vert z=2) &= (0,0,0.02,0.06,0.08,0.10,0.22,0.22,0.22,0.08),\\
p(d\vert z=3) &= (0.08,0,0,0.08,0.02,0.22,0.08,0.22,0.08,0.22),\\
p(d\vert z=4) &= (0.08,0.22,0.22,0.02,0.22,0,0.10,0,0.14,0).
\end{align*}

We use six training environments with
\begin{align*}
p(z\vert e_1)&=(0.49,0.01,0.01,0.49),&
p(z\vert e_2)&=(0.45,0.05,0.05,0.45),&
p(z\vert e_3)&=(0.40,0.10,0.10,0.40),\\
p(z\vert e_4)&=(0.30,0.20,0.20,0.30),&
p(z\vert e_5)&=(0.20,0.30,0.30,0.20),&
p(z\vert e_6)&=(0.10,0.40,0.40,0.10).
\end{align*}
These mixtures induce an association between color and label in the pooled training distribution: red images are predominantly assigned label $0$, while green images are predominantly assigned label $1$. Note that this association is not produced by an environment-specific coloring rule; rather, it emerges from the stable prototype mechanisms after marginalizing over the environment-dependent mixture $p(z\vert e)$.

For the test environment we set
\begin{equation*}
p(z\vert e_{\mathrm{test}})=(0.01,0.49,0.49,0.01).
\end{equation*}

\paragraph{Implementation details}
For each training environment we generate $5000$ examples, and for the test environment we generate $10000$ examples. 

All methods use the same convolutional feature extractor. The feature extractor consists of two convolutional blocks, each with a $3\times 3$ convolution, ReLU activation, and $2\times 2$ max pooling, followed by a fully connected layer with ReLU activation. The resulting representation dimension is $32$. 

For our method, the encoder $g_\theta(x,y,e)$ takes the image, binary label, and environment index as input and outputs the mean and log-variance of a one-dimensional Gaussian variational distribution. The decoder $f^{(0)}_\phi$ maps a latent sample $z$ to a binary logit for $y$. The auxiliary model $f^{(1)}_\varphi(x,e)$ maps an image and environment index to a binary logit estimating $p_\varphi(y=1\vert x,e)$. The environment classifier $h_\psi(x)$ outputs a categorical distribution over the six training environments. 

During training, for each minibatch observation $(x_i,y_i,e_i)$, we evaluate $g_\theta(x_i,y_i,e)$ for all six training environments $e\in\Etr$. We also evaluate $g_\theta(x_i,0,e)$ and $g_\theta(x_i,1,e)$ for all training environments in order to form
\begin{equation}
q_e(z\vert x_i)
=
p_\varphi(y=0\vert x_i,e)\,q_e(z\vert x_i,0)
+
p_\varphi(y=1\vert x_i,e)\,q_e(z\vert x_i,1).
\end{equation}
The estimates $\hat p_\psi(e\vert x_i)$ are obtained by applying a softmax to the output of $h_\psi(x_i)$, and
\begin{equation}
\hat p(e\vert x_i,y_i)
=
\frac{
\hat p_\psi(e\vert x_i)\, p_\varphi(y_i\vert x_i,e)
}{
\sum_{e'\in\Etr}
\hat p_\psi(e'\vert x_i)\, p_\varphi(y_i\vert x_i,e')
}.
\end{equation}

The posterior coupling term is estimated using one reparametrized sample from each $q_e(z\vert x_i,y_i)$. The prior coupling term is estimated using one sample from each mixture $q_e(z\vert x_i)$; this sample is obtained by first sampling the Bernoulli label component and then sampling from the selected Gaussian component. 

All models are trained for $10$ epochs with Adam, learning rate $10^{-3}$, batch size $128$, and gradient clipping at norm $5$. 
Cross-validation over the grid $\{0.01,0.1,1,10,100\}$, using validation NLL as the selection criterion, yielded
\begin{equation}
\lambda_{\mathrm{IRM}}=10, \qquad  \lambda_{\mathrm{VREx}}=10, \qquad \lambda_{\mathrm{Fishr}}=0.01.
\end{equation}

At test time, for each test image $x$, we compute $\hat p_\psi(e\vert x)$ for all training environments, construct $q_e(z\vert x)$ for each of them, and approximate
\begin{equation}
\hat p(y=1\vert x)
=
\sum_{e\in\Etr}
\hat p_\psi(e\vert x)
\int p_\phi(y=1\vert z)\,q_e(z\vert x) \, dz
\end{equation}
by Monte Carlo. In the final evaluation we use $1000$ samples from each Gaussian component. Accuracy is computed by thresholding $\hat p(y=1\vert x)$ at $0.5$, and calibration is measured using $10$-bin expected calibration error.

\subsection{Quasar-Star Classification}\label{sup:star_details}

We query spectroscopically confirmed objects from SDSS and retain only objects whose spectroscopic class is either \texttt{QSO} or \texttt{STAR}. The positive class corresponds to quasars and the negative class corresponds to stars.

We keep only objects with zero spectroscopic warning flag, clean photometry, primary photometric detections, and $r$-band magnitude between $14$ and $22$. 
We query at most $50{,}000$ objects before the subsequent class restriction to \texttt{QSO} and \texttt{STAR}. Rows with missing values in the retained variables are removed after feature construction. 

Each object has right ascension and declination in the ICRS coordinate system. We convert these coordinates to Galactic coordinates and compute the absolute Galactic latitude $\vert b_i\vert$. We then partition the objects into 5 quintile bins of $\vert b_i\vert$.
This construction induces environments corresponding to different sky regions, possibly leading to variation in stellar density and quasar prevalence. We select a shifted held-out environment according to this quasar prevalence.

From the ugriz magnitudes, we construct adjacent color features
\begin{equation}
u_i-g_i,\qquad
g_i-r_i,\qquad
r_i-i_i,\qquad
i_i-z_i,
\end{equation}
and additionally include the $r$-band magnitude.

We standardize the covariates using the training environments and apply the resulting transformation to the held-out environments.

Figure \ref{fig:dists_full} reports train and test standardized feature distributions.

\paragraph{Implementation details} All methods use the same two-layer feature extractor with width $16$, ReLU activations and representation dimension $8$.

We train all methods using minibatches of size $128$ and Adam with learning rate $10^{-3}$. 
Cross-validation over the grid $\{0.01,0.1,1,10,100\}$, using validation NLL as the selection criterion, yielded
\begin{equation}
\lambda_{\mathrm{IRM}} = \lambda_{\mathrm{VREx}} = \lambda_{\mathrm{Fishr}} = 0.01.
\end{equation}
For EmBER, held-out probabilities are estimated using Monte Carlo marginalization with $1000$ samples. 

\begin{figure}
    \centering
    \includegraphics[width=0.9\linewidth]{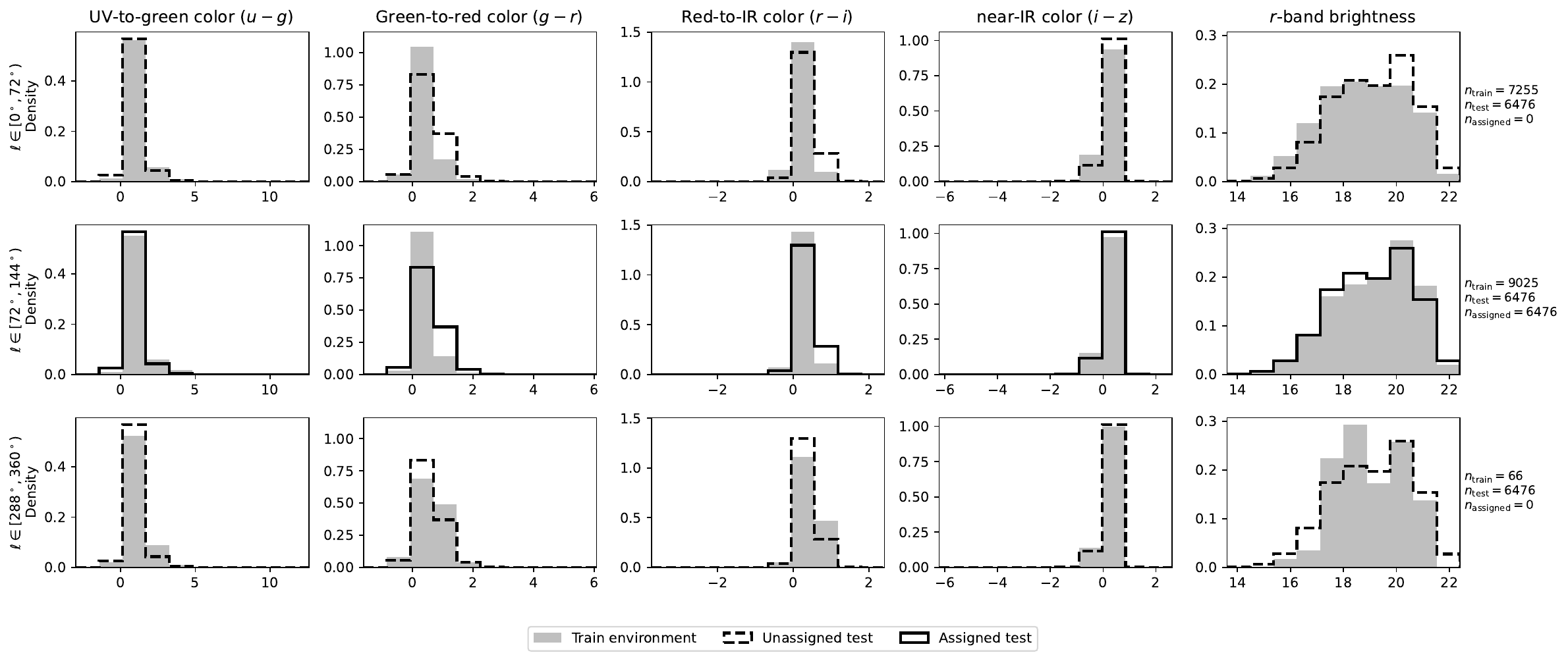}
    \caption{SDSS feature distributions by selected training environment. For all test points highest weight is assigned to environment $\ell\in[72^\circ,144^\circ)$, whose feature distribution closely matches the test distribution.}
    \label{fig:dists_full}
\end{figure}

\subsection{Microbiome classification} \label{sup:microbiom_details}

 We use the five colorectal-cancer cohorts available in the processed archive, namely \texttt{crc\_baxter}, \texttt{crc\_xiang}, \texttt{crc\_zackular}, \texttt{crc\_zeller}, and \texttt{crc\_zhao}.
 For each cohort, we load the sample metadata and the corresponding OTU table from the RDP-annotated output files.

 We exclude samples annotated as adenoma or polyp, so the task is restricted to cancer-versus-control classification.
 We preprocess the OTU tables by collapsing taxa to the genus level. For each sample, genus counts are converted to relative abundances, yielding one compositional feature vector per sample. We merge these features with the detected disease labels using the sample identifiers. After combining cohorts, missing genera are filled with zero abundance. We then retain only genera whose empirical prevalence across the combined dataset is at least $1\%$.

 We remove cohorts that do not contain both classes, resulting in exclusion of \texttt{crc\_zhao} which contains only controls. Evaluation is performed on one held-out cohort.
 We train on three cohorts, \texttt{crc\_zeller}, \texttt{crc\_xiang}, and \texttt{crc\_zackular}, and evaluate on the held-out cohort \texttt{crc\_baxter}. 

Before fitting the models, genus-abundance features are standardized using the training environments only, and the same transformation is applied to the held-out test environment. 

\paragraph{Implementation details}
All methods use the same linear feature map followed by a binary classifier, with representation dimension $16$. EmBER latent dimension was set to 1. We compare EmBER to ERM, IRM, VREx, and Fishr. All models are trained for $50$ epochs with Adam, learning rate $10^{-3}$, and batch size $128$.
Cross-validation over the grid $\{0.01,0.1,1,10,100\}$, using validation NLL as the selection criterion, yielded
\begin{equation}
\lambda_{\mathrm{IRM}}=10, \qquad \lambda_{\mathrm{V\text{-}REx}}=100, \qquad \lambda_{\mathrm{Fishr}}=100.
\end{equation}

For EmBER, held-out probabilities are estimated using Monte Carlo marginalization with $1000$ samples. 

For nearest-neighbor analysis, for each test observation $x_i$, we restrict the training data to the assigned environment $\hat e_i$ and find the nearest training observation
\begin{equation}
j(i)
=
\arg\min_{j:e_j=\hat e_i}
\Vert x_j-x_i\Vert_2.
\end{equation}
in the standardized feature space.

\subsection{Sepsis prediction}

We construct the sepsis experiment from the public PhysioNet 2019 Challenge training data. We use up to 15,000 patient files from each of the two public sources, \texttt{training\_setA} and \texttt{training\_setB}. Each file contains an hourly ICU time series for a single patient. We convert each file into one observation by summarizing the first 24 hours of the stay.

The covariates consist of demographic variables and summaries of early vital signs and laboratory measurements. For each time-varying measurement, we compute its last observed value, mean, minimum, maximum, and missingness fraction during the observation window.

We train on five source-unit environments and evaluate on a held-out source-unit environment. Specifically, the training environments are \texttt{A\_unit2}, \texttt{B\_unit1}, \texttt{A\_unit1}, \texttt{B\_unit2}, and \texttt{B\_unknown}, and the test environment is \texttt{A\_unknown}. 

After this split, we filter features using only the training environments, retaining features with at most 5\% missing values in the training data. Remaining missing values are imputed with the corresponding training-set median. We then standardize covariates using the training-set mean and standard deviation and apply the same transformation to the held-out test environment.

\paragraph{Implementation details}
All methods use the same representation network: a one-hidden-layer multilayer perceptron with hidden width 16 and output dimension 8. For EmBER, the latent dimension was set to 1. We train all methods for 50 epochs with minibatch size 128 and learning rate $10^{-3}$. 
Cross-validation over the grid $\{0.01,0.1,1,10,100\}$, using validation NLL as the selection criterion, yielded
\begin{equation}
\lambda_{\mathrm{IRM}}=10, \qquad \lambda_{\mathrm{VREx}}=10, \qquad \lambda_{\mathrm{Fishr}}=0.1.
\end{equation}

For EmBER, held-out probabilities are estimated using Monte Carlo marginalization with $1000$ samples.

\end{document}